%% file: main.tex
\documentclass[11pt]{article}

\input{arxiv}

\input{commands}

\title{Regret Lower Bounds for Learning Linear Quadratic Gaussian Systems}
\begin{document}

\maketitle

\input{sections/abstract.tex}
\newpage
\tableofcontents
\newpage
\input{sections/intro.tex}

\input{sections/RiccatiandRegret.tex}

\input{sections/fiandreg.tex}

\input{sections/statefeedback.tex}

\input{sections/partiallyobs.tex}

\input{sections/acknowledgements.tex}

\bibliographystyle{unsrtnat}
\bibliography{main.bib}

\appendix

\input{Appendix/proofsSF.tex}

\input{Appendix/VanTrees.tex}

\end{document}

%% file: arxiv.tex
\usepackage{authblk}
\author[1]{Ingvar Ziemann}
\author[2]{Henrik Sandberg}
\affil[1]{University of Pennsylvania}
\affil[2]{KTH Royal Institute of Technology}

\usepackage[utf8]{inputenc} 
\usepackage[T1]{fontenc} 

\usepackage{fullpage}

\usepackage{url}            
\usepackage{booktabs}       
\usepackage{amsfonts}       
\usepackage{nicefrac}       
\usepackage{microtype}      
\usepackage[dvipsnames]{xcolor}

\usepackage{natbib}
\usepackage{enumitem}
\usepackage{amsthm}
\usepackage{amssymb}
\usepackage{amsmath}
\usepackage{mathrsfs}
\usepackage{hyperref}
\hypersetup{
    pdftoolbar=true,        
    pdfmenubar=true,        
    pdffitwindow=false,     
    pdfstartview={FitH},    
    pdfnewwindow=true,      
    colorlinks=true,       
    linkcolor=blue,          
    citecolor=ForestGreen,        
    filecolor=magenta,      
    urlcolor=red,           
    breaklinks=false,
}

\usepackage{cleveref}
\usepackage{thmtools}
\usepackage{thm-restate}
\usepackage{graphicx}

\numberwithin{equation}{section}

\date{}

%% file: commands.tex
\newcommand{\e}{\varepsilon}

\newcommand{\E}{\mathbf{E}}

\DeclareMathOperator{\spn}{span}

\DeclareMathOperator{\tr}{tr}
\DeclareMathOperator{\VEC}{\mathsf{vec}}
\DeclareMathOperator{\dop}{\mathsf{D}}
\DeclareMathOperator{\I}{\mathtt{I}}
\DeclareMathOperator{\J}{\mathtt{J}}
\newcommand{\diff}{\mathtt{d}}
\DeclareMathOperator{\del}{\mathsf{\partial}}

\newcommand{\T}{\mathsf{T}}

\newcommand{\Z}{\mathbb{Z}}

\newcommand{\dx}{d_{\mathsf{X}}}
\newcommand{\dy}{d_{\mathsf{Y}}}
\newcommand{\dz}{d_{\mathsf{Z}}}

\newcommand{\du}{d_{\mathsf{U}}}

\newcommand{\R}{\mathbb{R}}
\newcommand{\N}{\mathbb{N}}
\renewcommand{\Pr}{\mathbf{P}}

\newcommand{\iid}{iid}

\newcommand{\opnorm}[1]{\| #1 \|_{\mathsf{op}}}
\newcommand{\bigopnorm}[1]{\left\| #1 \right\|_{\mathsf{op}}}

\newcommand{\sym}{\mathrm{sym}}

\newtheorem{definition}{Definition}[section] 
\newtheorem{theorem}{Theorem}[section] 
\newtheorem{proposition}{Proposition}[section] 
\newtheorem{corollary}{Corollary}[section] 
\newtheorem{remark}{Remark}[section]
\newtheorem{lemma}{Lemma}[section]

%% file: sections/abstract.tex
\begin{abstract}
    We establish regret lower bounds for adaptively controlling an unknown linear Gaussian system with quadratic costs. We combine ideas from experiment design, estimation theory and a perturbation bound of certain information matrices to derive regret lower bounds exhibiting scaling on the order of magnitude $\sqrt{T}$ in the time horizon $T$.  Our bounds accurately capture the role of control-theoretic parameters and we are able to show that systems that are hard to control are also hard to learn to control; when instantiated to state feedback systems we recover the dimensional dependency of earlier work but with improved scaling with system-theoretic constants such as system costs and Gramians. Furthermore, we extend our results to a class of partially observed systems and demonstrate that systems with poor observability structure also are hard to learn to control.
\end{abstract}

%% file: sections/intro.tex
\section{Introduction}

Learning algorithms are set to play an increasing role in modern engineering solutions. Early successes include walking robots \citep{yang2020data} and playing repeated games such as Go \citep{silver2017mastering} and are likely to become increasingly important in modern safety-critical infrastructure such as smart grids and intelligent transportation. However, their emergence in safety-critical systems is not without problems. Indeed, one of the hallmarks of these early successes is abundant data from a relatively unchanging source, potentially even through simulation access. By contrast, their failure modes as a component of modern engineering systems in dynamic, changing, environments when data is scarce remain poorly understood.

Such an understanding must necessarily be based on two components: \emph{i.} the study of fundamental performance limitations, which no algorithm can excede, and \emph{ii.} the provision of algorithms which match these fundamental limitations.  In this work, we focus on the first component and provide an information-theoretic framework for understanding the fundamental limits of adaptive controlling an a priori unknown linear Gaussian system subject to a quadratic cost function. We emphasize the interaction between the hardness of learning and the hardness of control by deriving regret lower bounds that accurately capture various system-theoretic constants: our bounds demonstrate that linear systems with poor controllability (or observability) structure are hard to learn to control.


\subsection{Problem Formulation}
 We study the fundamental limitations to adaptively learning to control  the following parametric system model:
\begin{equation}
\begin{aligned}
\label{eq:lds_ac}
X_{t+1}& = A(\theta) X_t + B(\theta) U_t + W_t, & X_0 \sim \mathsf{N}(0,\Gamma_{0}), & & t=0,1,\dots
\end{aligned}
\end{equation}
parameterized by $\theta\in \mathbb{R}^{d_\Theta}$---the joint law of 
the $X_t$ is parametrized by $\theta$.

We will also consider the extension to partially observed systems, in which the controller is constrained to rely on past and present observations of the form
\begin{align}
\label{eq:plds_ac}
Y_t&=C(\theta) X_t +V_t.
\end{align}
 The noise processes $W_t$ and $V_T$ are assumed mutually independent iid sequences of mean zero Gaussian noise with fixed covariance matrices $\Sigma_W\succeq 0$ and $\Sigma_V \succeq 0$ respectively.  Above $X_t \in \mathbb{R}^{\dx}$ and $Y_t \in \mathbb{R}^{\dy}$ respectively denote the observations available to the learner, which are obtained in a sequential fashion after applying the input $U_{t-1} \in \mathbb{R}^{\du}$ according to \eqref{eq:lds_ac} and \eqref{eq:lds_ac}-\eqref{eq:plds_ac} in the case of partially observed sytems.  The matrices $A(\theta) \in \mathbb{R}^{\dx\times\dx},B(\theta)\in\mathbb{R}^{\dx\times\du},C(\theta)\in \mathbb{R}^{\dy\times\dx}$ are assumed to be known continuously differentiable functions of the unknown parameter, $\theta \in \mathbb{R}^{d_\Theta}$. The unstructured case, to which much attention has been devoted in the literature, is recovered by setting $\VEC \begin{bmatrix} A(\theta) & B(\theta) & C(\theta) \end{bmatrix}=\theta$. We further denote by $\mathcal{Y}_t$ the sigma-field generated by the observations of $Y_1,\dots,Y_t$ and possible auxiliary randomization, $\mathtt{AUX}$, a random variable with density against the $d_{\mathtt{AUX}}$-dimensional Lebesque measure. We also point out that the state feedback (SF) case \eqref{eq:lds_ac} can be recovered from the general partially observable (PO) case \eqref{eq:lds_ac}-\eqref{eq:plds_ac} by setting $C(\theta)=I_{\dx}$ and $\Sigma_V=0_{\dx\times \dx}$. In this case, the nondegeneracy condition above simplifies to $\Sigma_W \succ 0$.

The learner is to design a policy $\pi$, constituted by a sequence of conditional laws of the variables $U_t$ given $\mathcal{Y}_t$. We will say that a policy $\pi$ is admissible if the process $\{U_t\}_{t=0,1,\dots}$ is adapted to the filtration $\{\mathcal{Y}_t\}_{t=0,1,\dots}$. Roughly, the learner's goal is to choose such an admissible policy $\pi$ as to minimize the cumulative cost 
\begin{align*}
\mathsf{V}_T^{\pi}(\theta) &\triangleq \E_\theta^\pi \left[ \sum_{t=0}^{T-1}\left(  X_t^\T Q X_t +U_t^\T R U_t\right) + X_T^\T Q_T(\theta) X_T \right],
\end{align*}
where we have fixed two known positive definite cost weighting matrices $Q,R \succ 0$, $Q \in \mathbb{R}^{\dx\times \dx}, R \in \mathbb{R}^{\du\times \du}$ and a terminal cost matrix $Q_T(\theta) \succeq 0$. An equivalent formulation is to minimize the cumulative  suboptimality due to not knowing the true parameter $\theta$, namely the regret:
\begin{equation}
\begin{aligned}
\label{regdef}
\mathsf{R}_T^\pi(\theta)&\triangleq \underbrace{\E_\theta^\pi \left[ \sum_{t=0}^{T-1}\left(  X_t^\T Q X_t +U_t^\T R U_t\right) + X_T^\T Q_T(\theta) X_T \right]}_{\mathsf{V}_T^\pi(\theta)}\\
&-\underbrace{\inf_{\pi}\E_\theta^\pi \left[ \sum_{t=0}^{T-1}\left(  X_t^\T Q X_t +U_t^\T R U_t\right) + X_T^\T Q_T(\theta) X_T \right]}_{\mathsf{V}_T^{\star}(\theta)}.
\end{aligned}
\end{equation}
Whenever $\theta = \VEC [A \ B \ C]$, we will also allow ourselves the  abuse of notation $ \mathsf{R}_T^\pi\left( A , B , C \right)=\mathsf{R}_T^\pi(\theta)$.

In principle, there is no learning involved in the above formulation since the optimal policy $\pi_\star(\theta)$ is admissible. As is standard in modern statistics, we circumvent this issue by requiring that the learner performs well on a neighborhood of instances. Namely, for a fixed tolerance $\e>0$ we posit that the learner seeks to minimze $\sup_{\theta' \in B(\theta,\e)}\mathsf{R}_T^\pi(\theta')$. It is this latter quantity which we seek to study in terms of fundamental achievable performance limits.

Finally, we make the following standing assumptions about the system \eqref{eq:lds_ac}-\eqref{eq:plds_ac}: 
 \begin{enumerate}[label=\textnormal{A\arabic*}]
 \item[A1.] The  pair $(A(\theta),B(\theta))$ is stabilizable and the pair $(A(\theta),C(\theta))$ is detectable. 
 \item[A2.] The terminal cost $Q_T$ renders the optimal controller stationary; $Q_T(\theta) =P(\theta)$ where $P(\theta)$ is the solution to the discrete algebraic Riccati equation (which is reproduced in \eqref{pDARE}).
 \item[A3.] The distribution of $X_0$ renders the optimal filter stationary; $X_0 \sim \mathsf{N}(0,S(\theta))$ where $S(\theta)$ is given by \eqref{sDARE}.
 \end{enumerate}
 
 The first assumption, A1., guarantees the feasibility of the long-run averaged version of the problem at hand. Assumptions A2. and A3. are made to streamline the exposition; it is possible to derive analogous results in their absence, the only difference being that the quantities related to the Riccati equations (\ref{pDARE})  and (\ref{sDARE}) in the regret representation derived in Lemma~\ref{regretlemma} become time-varying. However, the time-varying versions of these quantities converge at an exponential rate to those used here, so the overall difference is negligible.

\subsection{Contribution}

We establish fundamental limits for adaptive control problems by proving local minimax regret lower bounds for LQR and LQG problems. Our lower bounds take the form
\begin{align}
\label{localminimaxregretcontribution}
 \inf_{\pi} \sup_{\theta' \in B(\theta,\e)  } \mathsf{R}_T^\pi(\theta') \gtrsim c(\theta,\e) \sqrt{T} 
\end{align}
for a constant $c(\theta,\e)$ that captures the instance-specific hardness of the problem, such as dependence on the dimension of the unknown parameter, $d_\theta$, and various system quantities depending on the controllability and observability structure of the instance under consideration. We remark the radius $\e$ appearing in \eqref{localminimaxregretcontribution} can be thought of as the strength of an adversary selecting the parameters against which the adaptive algorithm is to compete---we will prove slightly stronger statements in which the strength of the adversary, $\e$, is vanishing: $\e\to 0$. Put differently, there exists no controller that can perform better than the fundamental limit \eqref{localminimaxregretcontribution} uniformly in an infinitesimal neighborhood in parameter space.

We match the dimensional dependencies in prior work \citep{simchowitz2020naive}, but in contrast to it, we obtain much sharper dependencies on system-theoretic constants. This allows us to conclude that systems with poor controllability structure are much harder to learn to control. Moreover, we extend our results to a particular class of partially observed systems and show that poor observability can be analogously detrimental to learning performance. 

At a more technical level, we also introduce a new proof approach for establishing regret lower bounds. Our technique relies on linking the nullspace of the Fisher information matrix (\Cref{eq:fisherdef}) of the data collected by any policy to the regret incurred by that policy. Informally, for any policy $\pi$, this idea can be summarized as:
\begin{align*}
\textnormal{Fisher Information}(\pi) = \textnormal{Fisher Information}(\pi^\star) + O(\textnormal{Regret})
\end{align*}
where $\pi^\star$ is the optimal policy. In other words, policies which generate experiments much more informative than the optimal one must have a significant regret component. The Fisher information thus allows us to make precise the exploration-exploitation trade-off in adaptive control. We derive our lower bound \Cref{localminimaxregretcontribution} by combining this insight with a version of Van Trees' Inequality (\Cref{VTineq}).

\subsection{Related Work}

Regret minimization in the context of linear quadratic systems was first introduced by  \cite{lai1986asymptotically} and \cite{lai1986extended}, following their treatment of regret in the related multi-armed bandit problem \citep{lai1985asymptotically}, see also \cite{guo1995convergence}. The notion of regret captures that there is a trade-off between performance and information collected. This is what Feldbaum called the dual nature of control, \cite{feldbaum1960dual1,feldbaum1960dual2} or what is now known in the reinforcement learning literature as the exploration-exploitation trade-off. On the one hand, one wishes for the algorithm to perform well now, but on the other, one also needs the algorithm to be informative about the state of world, as to be able to reject for instance model misspecification or other uncertainty. Regret lower bounds describe the trade-off between the statistical rate of convergence of an adaptive algorithm and its performance on a nominal instance. Generally, such lower bounds are the consequence of requiring an algorithm to perform well on some class of models, subject to uncertainty about the nominal instance. Such an algorithm must necessarily ``explore'' to determine which model to optimize for. In turn,  such exploration leads to sub-optimal performance on the nominal instance. Key algorithmic principles based on this need for exploration date back at least to Simon's 1956 introduction of certainty equivalence \cite{simon1956dynamic} and the notion of dual control \citep{feldbaum1960dual1,feldbaum1960dual2}. When it comes to linear quadratic systems, an early reference is the self-tuning regulator by \cite{aastrom1973self}, in turn inspired by Kalman's earlier work \citep{kalman1958design}. In this context, the primary mathematical issue was first and foremost the convergence of the adaptive controller to the global optimum \citep{goodwin1981discrete, lai1982least, becker1985adaptive, campi1998adaptive}. These works provide and analyze adaptive algorithms that are asymptotically optimal on average, which one now would perhaps simply call sublinear regret.

Recently, the adaptive linear-quadratic-Gaussian problem has mainly been studied under the assumption of perfect state observability $(C=I_{\dx}, V_t =0$ identically). While the problem has a rich history, it was re-popularized by \cite{abbasi2011regret} in which the authors provide an algorithm attaining $O(\sqrt{T})$ regret. A number of works following that publication focus on improving and providing more computationally tractable algorithms in this setting \citep{ouyang2017control, dean2018regret, abeille2018improved, abbasi2019model, mania2019certainty, cohen2019learning, faradonbeh2020input,abeille2020efficient, jedra2021minimal}. While the emphasis of these works is entirely on providing upper bounds, recently, some effort has been made to understand the complexity of the problem in terms of lower bounds. Notably,  \cite{simchowitz2020naive} provides nearly matching  upper and lower bounds scaling almost correctly with the dimensional dependence given that the entire set of parameters $(A,B)$ are unknown. While \cite{simchowitz2020naive} provide lower bounds that scale correctly with the dimension of the problem \eqref{eq:lds_ac}, their bounds are rather loose in terms of system-theoretic quantities. In many situations, these can be exponentially larger than the relevant dimensional factors \citep{tsiamis2021linear}. With this in mind, our work seeks to understand the hardness of learning to control in terms of such system-theoretic quantities. To this end, we provide refined lower bounds for the state-feedback setting, which further apply to partially observed systems. Later, in our work \cite[Theorem 8]{tsiamis2022learning}, we leveraged the refined bounds of the present work to show that these "hidden" system-theoretic quantities can have exponential impact on the regret for many reasonable classes of systems. Indeed, there is still a rather large gap to be filled in the literature in terms of how control-theoretic parameters (such as gramians) affect the  scaling limits of the smallest possible regret an algorithm can incur. Motivated by this gap, we develop a theory of regret lower bounds based on Cramér-Rao type bounds which are known to yield tight lower bounds in system identification. We refer to \Cref{table:adaptsum} for schematic overview of the above mentioned results.

{
\begin{table*}
\center
\caption{Summary of Results: Regret Minimization in Adaptive Control}
\label{table:adaptsum}
\begin{tabular}{|c |c|c|c|c| } 
 \hline
 Paper & Setting & Method & Upper Bound & Lower Bound \\ 
 \hline
\cite{abbasi2011regret}  & SF & Optimism &$ \tilde O(\sqrt{T})  $ &   \\ 
\hline
\cite{dean2018regret}  &SF & CE& $\tilde O(T^{2/3})$ &  \\ 
 \hline 
 \cite{faradonbeh2020input} & & CE &  &\\
 \cite{mania2019certainty} &SF & CE  & $ \tilde O(\sqrt{T})$ &  \\ 
 \cite{cohen2019learning} &  &Optimism& &  \\ 
 \hline
 \cite{simchowitz2020naive} &SF  & CE & ${O( \sqrt{\dx \du^2 T})}$ & ${ \Omega( \sqrt{\dx \du^2 T})}$ \\ \hline
 \cite{simchowitz2020improper}&PO &Gradient& $\tilde O(\sqrt{T})$ &  \\ 
 \hline
  \Cref{sec:sf} &SF & & & $\Omega(\sqrt{\dx \du^2 T})$  \\
\hline
   \Cref{sec:po}&PO & & & $\Omega(\sqrt{T})$  \\
 \hline
   \cite{tsiamis2022learning} &SF &  & $  $ & $\Omega\left(\sqrt{\frac{1}{\dx}2^{\kappa} T}\right)$  \\ 
 \hline
\end{tabular}
\end{table*}
}

Turning to the more general situtation including partial state observability, the key references are \cite{simchowitz2020improper} and \cite{lale2020logarithmic}. \cite{simchowitz2020improper} consider a closely related but more general setting in which the noise model is (possibly) adversarial instead of Gaussian and produce a $\tilde O (\sqrt{T})$ regret algorithm.   \cite{lale2020logarithmic} consider a system of the form considered in the present work and give an algorithm attaining $\tilde O(T^{2/3})$ regret. However, they also show that when a condition related to the persistency of excitation of the benchmark law holds, their algorithm can attain polylogarithmic regret. However, it is not known when their condition holds for the optimal law. In other words, their notion of regret may differ from ours.  We will construct a negative result in this direction that satisfies A1-A3 in \Cref{ch:regretpo}; we show that there exists stabilizable and detectable systems with full rank noise on which obtaining logarithmic regret is impossible. At this point, we remark that  bounds on the scale $\log T$ have been well-known for some time \citep{nemirovski1984optimal, lai1986asymptotically, raginsky2010divergence, rantzer2018adaptive}. By contrast, lower bounds on the order $\sqrt{T}$ are a more recent phenomenon starting with \cite{simchowitz2020naive,cassel2020logarithmic,ziemann2020phase}.

In principle, the recognition that $\sqrt{T}$ regret is often unimprovable stems from the fact that the stochastic adaptive control problem is intimitely connected to parameter estimation. Already from the outset, algorithm design has to a large extent been based on certainty equivalence; that is, estimating the parameters and plugging these estimates into an optimality equation, as if they were the ground truth \citep{aastrom1973self, mania2019certainty}. Our lower bound condition is also closely related to parameter estimation. It is a consequence of the viewpoint that an adaptive controller, to be asymptotically optimal, must generate an experiment asymptotically very similar to one in which the optimal controller has generated the data. If this experiment is ``bad'' in a certain sense, logarithmic regret becomes impossible. This reasoning is akin to earlier results in experiment design and identification for control. In particular, there is a very interesting result due to \cite{gevers1986optimal} which finds that the optimal experiment for minimum variance control is to use the minimum variance controller itself. This is the  opposite of the phenomenon which more general adaptive linear-quadratic regulation problems exhibit, as noted for instance by \cite{lin1985will} and \cite{polderman1986necessity}. Here application of the optimal feedback law typically yields a singular experiment. However, even under more general circumstances, it still holds true that the optimal experiment design is  closed-loop \citep{hjalmarsson1996model}. For more on experiment design, we refer the reader to the book \cite{pukelsheim2006optimal}. 

Our lower bound condition, \emph{uninformativeness} (\Cref{uninfdef}), is related to identifiability. Actually, it is inspired by a similar phenomenon in point estimation, which may become arbitrarily hard when the Fisher information is singular \citep{rothenberg1971identification, goodrich1979necessary, stoica1982non, stoica2001parameter}. Indeed, we will see that uninformativeness allows one to draw conclusions analogous to results of \cite{polderman1986necessity} about the necessity of identifying the true parameter in adaptive control  subject to the data being generated by the optimal controller itself (which one would hope to, at least asymptotically, be close to). In principle, the condition stipulates that there is a lack of curvature (as measured by the loss geometry) in certain directions of parameter space in a neighborhood of the optimal policy. It is precisely this lack of curvature that prevents logarithmic regret, which has been reported in certain special cases of the adaptive LQG problem \citep{lai1986asymptotically, guo1995convergence}.

The proof approach for our lower bound also relies on methods pioneered in parameter estimation;  we use Van Trees' inequality \citep{van2004detection,bobrovsky1987some, gill1995applications}. This necessarily involves the Fisher information, which, quite naturally, allows for taking problem structure into account by considering different parametrizations of the problem dynamics. We also note that the idea to bound a minimax complexity by a suitable family of Bayesian problems is well-known in the statistics literature \citep{gill1995applications}. See also \cite{van2000asymptotic,tsybakov2008introduction, ibragimov2013statistical} and the references therein. We also note that \cite{raginsky2010divergence} attributes the idea to use Cramér-Rao type bounds to derive performance lower bounds in adaptive control to \cite{nemirovski1984optimal}.\footnote{The source is in Russian, and we have not been able to verify it.} Finally, we note in passing that non-singularity of the Fisher information is strongly related to the size of the smallest singular value of the empirical covariance matrix, which has been the emphasis of some recent advances in linear system identification \citep{faradonbeh2018finite, pmlr-v75-simchowitz18a,sarkar2019near,jedra2020finite,wagenmaker2021task}. In the present work, we ask analogous questions in the regret-minimization setting, in which there is a rather rich interplay between identification and control.

\subsection{Notation}
Maxima (resp.\ minima) of two numbers $a,b\in \R$ are denoted by $a\vee b =\max(a,b)$ ($a\wedge b = \min(a,b)$). For two sequences $\{a_t\}_{t\in \Z}$ and $\{b_t\}_{t\in \Z}$ we introduce the shorthand $a_t \lesssim b_t$ if there exists a universal constant $C>0$ and an integer $t_0$ such that $a_t \leq C b_t$ for every $t \geq t_0$.  Let $\mathsf{X} \subset \R^d$ and let $f,g \in \mathsf{X} \to R$. We write $f=O(g)$ if $\limsup_{x\to x_0} |f(x)/g(x)|<\infty$, where the limit point $x_0$ is typically understood from the context. We use $\tilde O$ to hide logarithmic factors and write  $f=o(g)$ if $\limsup_{x\to x_0} |f(x)/g(x)|=0$. We write  $f=\Omega(g)$ if $\limsup_{x\to x_0} |f(x)/g(x)|>0$. 

\paragraph{Euclidean Spaces}
The Euclidean norm on $\mathbb{R}^{d}$ is denoted $\|\cdot\|_2$,
and the unit sphere in $\R^d$ is denoted $\mathbb{S}^{d-1}$. The ball of radius $\e$ in $\|\cdot\|_2$, centered at $x \in \R^d$ is denoted $B(x,\e)$. The standard inner product on $\R^{d}$ is denoted $\langle\cdot,\cdot\rangle$. We embed matrices $M \in \R^{d_1\times d_2}$ in Euclidean space by vectorization: $\VEC M \in \mathbb{R}^{d_1 d_2}$, where $\VEC$ is the operator that vertically stacks the columns of $M$ (from left to right and from top to bottom). For a matrix $M$ the Euclidean norm is the Frobenius norm, i.e., $\|M\|_F\triangleq \|\VEC M\|_2$. We similarly define the inner product of two matrices $M,N$ by $\langle M, N \rangle \triangleq \langle \VEC M, \VEC N \rangle$. The transpose of a matrix $M$ is denoted by $M^\T$ and $\tr M $ denotes its trace. We define (twice) the symmetric component of $M$ by $\sym (M) \triangleq M+M^\T$. For a matrix $M \in \R^{d_1 \times d_2}$, we order its singular values $\sigma_{1}(M),\dots,\sigma_{d_1 \wedge d_2}(M)$ in descending order by magnitude. We also write $\opnorm{M}$ for its largest singular value: $\opnorm{M} \triangleq \sigma_1(M)$. To not carry dimensional notation, we will also use $\sigma_{\min}(M)$ for the smallest nonzero singular value.   For square matrices $M\in \R^{d\times d}$ with real eigenvalues, we similarly order the eigenvalues of $M$ in descending order as $\lambda_{1}(M),\dots,\lambda_{d}(M)$. In this case, $\lambda_{\min}(M)$ will also be used
to denote the minimum (possibly zero) eigenvalue of $M$. For two symmetric matrices $M, N$, we write $M \succ N$ ($M\succeq N)$ if $M-N$ is positive (semi-)definite.

\paragraph{Differentiation and Integration}
The set of $k$-times continuously differentiable functions on a subset $U$ of $\mathbb{R}^{d}$ is denoted $\mathscr{C}^k(U)$. To restrict attention to those functions in $\mathscr{C}^k(U)$ which are compactly supported, we write $\mathscr{C}_c^k(U)$. We use $\dop$ for Jacobian, $\diff$ for differential and $\nabla$ for the gradient. Expectation (resp.\ probability) with respect to all the randomness of the underlying
probability space is denoted by $\E$ (resp.\ $\Pr$). 


%% file: sections/RiccatiandRegret.tex
\section{Riccati, Regret and a Reduction to Bayesian Estimation}

We begin by recalling a number of elementary facts regarding the optimal control and filtering of the system \eqref{eq:lds_ac}-\eqref{eq:plds_ac}, valid in the case the parameter $\theta$ is fixed and A1-A3 hold. For a reference, see for instance \cite{soderstrom2002discrete}.

\paragraph{Riccati Equations.}
The expression for the linear system \eqref{eq:lds_ac}-\eqref{eq:plds_ac} and the regret (\ref{regdef}) are cumbersome to work with directly. However, these can be simplified using Riccati equations. Let us now recall, provided that $\theta=(A,B)$ is stabilizable, that the optimal policy minimizing the long-run average of $\mathsf{V}_T^\pi(\theta)$, is represented by a stabilizing feedback matrix $K(\theta)$ and can be expressed via $P(\theta)$ which together satisfy 
\begin{align}
\label{pDARE}
P &=Q +A^\T P_K A - A^\T PB (B^\T P B+R)^{-1}B^\T P A\\
\label{kDARE}
K &=-(B^\T P B+R)^{-1} (B^\T P A)
\end{align}
where dependence on $\theta$ has been omitted in \eqref{pDARE}-\eqref{kDARE} for notational brevity.

Since $X_t$ is not always directly  observed  in our problem formulation, (\ref{pDARE}) and (\ref{kDARE}) do not constitute complete solutions of the LQG problem. Indeed, the asympotically optimal policy is of the form $U_t = K\hat X_t$ where $\hat X_t= \E_\theta^\pi[X_t|\mathcal{Y}_t]$, which provided suitable initial conditions (to be specified momentarily), can be expressed recursively:
\begin{align}
\label{stateest}
\hat X_{t+1}& = A(\theta)\hat X_t + B(\theta)U_t + F(\theta) [Y_{t+1}-C(\theta)(A(\theta)\hat X_t+B(\theta)U_t)]
\end{align}
where $F(\theta)\in \mathbb{R}^{\dx\times \dy}$ is given by (\ref{lDARE}) and is characterized below in terms of a Riccati equation (\ref{sDARE}), dual to  (\ref{pDARE}) and (\ref{kDARE}). We further denote 
\begin{align}
\label{filteredstatenoise}
\nu_t =  F(\theta) [Y_{t+1}-C(\theta)(A(\theta)\hat X_t+B(\theta)U_t)],
\end{align}
which plays a role corresponding to that of $w_t$ in \eqref{eq:lds_ac} for the filtered state (\ref{stateest}). The process $\nu_t$ is iid Gaussian with mean zero and we denote its covariance $\Sigma_\nu$. It will also be convenient to introduce the $1$-step ahead prediction $\zeta_t= \E_\theta^\pi[X_t|\mathcal{Y}_{t-1}]$ which similarly satisfies a recursion
\begin{align}
\label{statepred}
\zeta_{t+1}&= A(\theta)\zeta_t + B(\theta)U_t + F(\theta)[Y_t-C(\theta)\zeta_t]
\end{align}
The quantity $F$ appearing in both the asymptotic Kalman filter recursions (\ref{stateest}) and (\ref{statepred}) is characterized by the Filter Riccati equation
\begin{align}
\label{sDARE}
S &= ASA^\T-ASC^\T(CSC^\T + \Sigma_V)^{-1}CSA^\T+\Sigma_W\\
\label{lDARE}
F &= S C^\T(C SC^\T+\Sigma_V)^{-1}
\end{align}
where we again omit dependence on $\theta$ in \eqref{sDARE}-\eqref{lDARE} for notational  brevity.

The quantity $S(\theta)$ is the (steady state) covariance matrix of $X_{t}-\zeta_t$. Similarly, we define $\Xi(\theta)$ to be the covariance matrix of $X_t-\hat X_t$, which can be expressed in terms of $S(\theta)$ as 
\begin{align*}
\Xi  = S-SC^\T (CSC^\T+\Sigma_V)^{-1}CS.
\end{align*}
Finally, we point out that unrolling the filter dynamics \eqref{stateest} allows us to define a sequence of $\theta$-parametrized linear maps $\mathbf{G}_t(\theta) : \mathbb{R}^{(\du+\dy)\times (t+1)}\to \mathbb{R}^{\dx}$ that act on past inputs and outputs to produce the present state estimate. In other words, $\mathbf{G}_t(\theta)$ is defined by
\begin{align}
\label{eq:linmapkalman}
    \mathbf{G}_t(\theta)(U_{0:t},Y_{0:t}) =\hat X_t
\end{align}
where, for a fixed $\theta$, $\hat X_t$ is given by \eqref{stateest}. We now turn to representing the regret (\ref{regdef}) in terms of the filtered state (\ref{stateest}).

\paragraph{Regret Representation.}
In terms of the quantites above, it is straightforward to verify that the optimal cost $\mathsf{V}_T^{\star}(\theta)$ can be expressed as \citep[see e.g.][Theorem 11.3]{soderstrom2002discrete}:
\begin{align}\label{eq:optcost}
\mathsf{V}_T^{\star}(\theta)=\E_\theta X_0^\T P(\theta) X_0+T\tr (\Sigma_\nu(\theta) P(\theta)) + T \tr (Q S(\theta)).
\end{align}
We now provide an alternative representation of the regret (\ref{regdef}).
\begin{lemma}
\label{regretlemma}
Assume A1-A3. Then:
\begin{equation}
\begin{aligned}
\label{regdef2}
\mathsf{R}_T^\pi(\theta) =\sum_{t=1}^{T-1} \E_{\theta}^\pi (U_t- K(\theta)\E_\theta^\pi [X_t | \mathcal{Y}_t])^\T (B^\T(\theta) P(\theta)B(\theta)+R) (U_t-K(\theta)\E_\theta^\pi [X_t | \mathcal{Y}_t]).
\end{aligned}
\end{equation}
\end{lemma}

\begin{proof}
After subtracting constant terms independent of $\pi$, this is immediate by Lemma 11.2 in \cite{soderstrom2002discrete}, which gives a general expression for $\mathsf{V}_T^\pi$. 
\end{proof}

\begin{remark}
For systems with an observed state described by \eqref{eq:lds_ac}, by embedding them into the description \eqref{eq:lds_ac}-\eqref{eq:plds_ac}, setting $C=I_{\dx}$ and $\Sigma_V = 0$, (\ref{regdef2}) becomes
\begin{align*}
\mathsf{R}_T^\pi(\theta) = \sum_{t=1}^T \E_{\theta}^\pi (U_t- K(\theta)X_t )^\T (B^\T(\theta) P(\theta)B(\theta)+R) (U_t-K(\theta)X_t ).
\end{align*}
\end{remark}

The problem of regret minimization in LQG can be seen as that of sequentially learning  $\theta\mapsto  K(\theta)  \E_\theta^\pi [X_t | \mathcal{Y}_t]$ which is the composition of both the asymptotically optimal state-feedback controller and the Kalman filter.

With this in mind, (\ref{regdef2}) may be understood as saying that regret minimization is at least as hard as minimizing a cumulative weighted quadratic estimation error for the sequence of estimands $ K(\theta)  \E_\theta^\pi [X_t | \mathcal{Y}_t]$. To be clear, our perspective is that we wish to estimate the function value of $ K(\theta)  \E_\theta^\pi [X_t | \mathcal{Y}_t]$ where the function to be estimated $ \theta\mapsto K(\theta)  \E_\theta^\pi [X_t | \mathcal{Y}_t]$ is revealed at time $t$. By relaxing the local minimax problem to a Bayesian setting, the entire trajectory $Y_{0:t}$ is then interepreted as a noisy observation of the underlying parameter $\theta$. A natural approach for variance lower bounds is to rely on Fisher information and use (Bayesian) Cram{\'e}r-Rao bounds. The discussion above is summarized by the following result.

\begin{lemma}
\label{lem:regretlbrelaxed}
Fix $\e>0$ and let $\mu$ be a smooth and compactly supported prior on $B(\theta,\e)$. Fix also a function $\tau: \mathbb{Z}\to \mathbb{Z}$ satisfying $\tau(t)\geq t, t\in \mathbb{N}$. As long as A1-A3 hold, it holds for any admissible policy $\pi$ that:
\begin{equation}
\begin{aligned}
\label{eq:regretlbrelaxed}
    &\sup_{\theta'\in  B(\theta,\e)}\mathsf{R}_T^\pi(\theta')\geq  \E_{\Theta \sim \mu} \mathsf{R}_T^\pi(\Theta) \\
    &\geq \sum_{t=0}^{T-1}\E_{\Theta \sim \mu}\E_{\Theta}^\pi \Big[ \left( (\E^\pi[K(\Theta)\mathbf{G}_t(\Theta)|\mathcal{Y}_{\tau(t)}]-K(\Theta)\mathbf{G}_t(\Theta))(U_{0:t},Y_{0:t})\right)^\T N_\star \\
    &\times \left( (\E^\pi[K(\Theta)\mathbf{G}_t(\Theta)|\mathcal{Y}_{\tau(t)}]-K(\Theta)\mathbf{G}_t(\Theta))(U_{0:t},Y_{0:t})\right)\Big]
\end{aligned}
\end{equation}
Where $\mathbf{G}_t(\theta)$ is defined by \cref{eq:linmapkalman} and $N_{\star}$ is any matrix satisfying $N_{\star} \preceq (B^\T(\theta') P(\theta')B(\theta')+R)$ for all $\theta' \in B(\theta,\e)$.
\end{lemma}

\begin{proof}
By virtue of \eqref{regdef2} we have that
\begin{align*}
&\E_{\Theta \sim \mu}\mathsf{R}_T^\pi(\Theta) \\
&= \E_{\Theta \sim \mu}\sum_{t=1}^{T-1} \E_{\Theta}^\pi (U_t- K(\Theta)\E_\Theta^\pi [X_t | \mathcal{Y}_t])^\T (B^\T(\Theta) P(\Theta)B(\Theta)+R) (U_t-K(\Theta)\E_\Theta^\pi [X_t | \mathcal{Y}_t])\\
&\geq \E_{\Theta \sim \mu}\sum_{t=1}^{T-1} \E_{\Theta}^\pi (U_t- K(\Theta)\E_\Theta^\pi [X_t | \mathcal{Y}_t])^\T  N_\star (U_t-K(\Theta)\E_\Theta^\pi [X_t | \mathcal{Y}_t])\\
&\geq \sum_{t=1}^{T-1} \inf_{\bar U_t\in \mathcal{Y}_{\tau(t)}} \E_{\Theta \sim \mu} \E_{\Theta}^\pi (\bar U_t- K(\Theta)\E_\Theta^\pi [X_t | \mathcal{Y}_{t}])^\T  N_\star (\bar U_t-K(\Theta)\E_\Theta^\pi [X_t | \mathcal{Y}_{t}]).
\end{align*}
Since $N_\star$ does not depend on $\Theta$, the minimization problem on the right hand side above is solved by the conditional mean of $K(\Theta)\E_\Theta^\pi [X_t | \mathcal{Y}_{t}]$ given $\mathcal{Y}_{\tau(t)}$.  The result then follows by noting $ K(\Theta)\E_\Theta^\pi [X_t | \mathcal{Y}_{\tau(t)}]=K(\Theta)\mathbf{G}_t(\theta)(U_{0:t},Y_{0:t})$ and "taking out what is known".
\end{proof}

In the state feedback setting \Cref{lem:regretlbrelaxed} reduces to
\begin{multline}
\label{eq:regretlbrelaxedSF}
 \sup_{\theta'\in B(\theta,\e)}\mathsf{R}_T^\pi(\theta')\\
    \geq\sum_{t=0}^{T-1} \E_{\Theta \sim \mu}\E_{\Theta}^\pi \tr \Big[N_\star  \left( \E[K(\Theta)|\mathcal{Y}_{\tau(t)}]-K(\Theta)\right) \left(X_tX_t^\T \right) \left( \E[K(\Theta)|\mathcal{Y}_{\tau(t)}]-K(\Theta) \right)^\T \Big].
\end{multline}

%% file: sections/fiandreg.tex
\section{Experiment Design and Regret}

Let us reiterate the point made above following \Cref{regretlemma} that we intend to lower-bound the regret by estimation-theoretic methods. From this perspective, the adaptive policy---the decision variable of the learner---generates an experiment, namely a sequence of input-output pairs $(Y_{0:{T-1}},U_{0:T-1})$.

To make precise the notion of an experiment---and the information contained in such an experiment---let us introduce the concept of Fisher information. This quantity can be thought of as a form of signal-to-noise ratio of the observed measurements with respect to the unknown parameter $\theta$. We are interested in the Fisher information pertaining to the information available to the learner in the setting \eqref{eq:lds_ac} where the measurement is the input-output pair $(Y_{0:{T-1}},U_{0:T-1})$. Denote by $\mathtt{p}_{\pi,\theta,T}$ the joint density of the random variable $(\mathtt{AUX},Y_0,\dots,Y_{T-1})$ under policy $\pi$ (conditionally on the the parameter $\theta$). The following information quantity serves as the basis for our analysis and is a policy-dependent measure of information available to the learner about the uncertain parameter $\theta$:
\begin{align}
\label{policyinformation}
\I(\theta;\pi,T) \triangleq \I_{\mathtt{p}}(\theta)
\end{align}
with $\I_{\mathtt{p}}(\theta)$ as in (\ref{eq:fisherdef}).

\subsection{Optimal Policies and Degenerate Experiments}

Naively, the perspective discussed in following Lemma~\ref{regretlemma} viewing (\ref{regdef2}) as a cumulative estimation error suggests a lower bound on the scale $\log T$, since one might think that (\ref{policyinformation}) should scale linearly in time, $T$. In this case, the associated parameter estimation errors variances should decay as $1/T$. Indeed, this is the case of \emph{certain} instances, see e.g., \cite{lai1986asymptotically}.

However, when the Fisher information corresponding to the experiment of running the optimal policy is singular, this reasoning fails. We will see that the Fisher information corresponding to low regret algorithms have nearly singular information.   Namely, if $\I(\theta;\pi_\star,T)$ given by  (\ref{policyinformation}) is singular, and this singularity is relevant for identifying $K(\theta)$, we expect there to be a non-trivial trade-off between exploration and exploitation. For instance, we will see that when the experiment corresponding to the optimal policy is degenerate, any algorithm with $O(\sqrt{T})$-regret necessarily generates an experiment in which the smallest (relevant) singular value only scales as $\sqrt{T}$; \eqref{policyinformation} scales sublinearly in certain relevant directions. Put yet differently:  if the optimal policy $\pi_\star$ yields an experiment in which $K(\theta)$ is not locally identifiable, we expect the regret  to be $\Omega(\sqrt{T})$.

We now make precise the above reasoning by imposing two conditions, which together rule out the possibility of  logarithmic regret. The first condition states that the optimal policy $\pi_\star$ does not persistently excite the parameters for local identifiability in terms of Fisher infomation. 

\begin{definition}
\label{uninfdef}
Fix $\e>0$ and a subspace $\mathtt{U}$ of $\mathbb{R}^{d_\Theta}$. The instance $(\theta, A(\cdot),B(\cdot),C(\cdot), Q,R,\Sigma_W,\Sigma_V)$ is $(\mathtt{U},\e)$-locally uninformative if there exists a neighborhood $B(\theta,\e)$ such that for all $\tilde \theta \in B(\theta,\e)$ and $ v \in \mathtt{U}$:
\begin{itemize}
\item  $\I(\tilde \theta;\pi_\star(\theta),T) v=0$  for all $T$; and
\item  $[\dop_\theta \VEC  K(\theta)] v \neq 0$ if $v \neq 0$.
\end{itemize}
\end{definition}
Any subspace $\mathtt{U} \subset \mathbb{R}^{d_\Theta}$, with all nonzero $v\in \mathtt{U}$ satisfying the above condition and of maximal dimension (i.e. largest possible satisfying the constraints), is called a (control) information singular subspace. The condition requires that the optimal policy pertaining to the instance $\theta$ does not persistently excite any instance in a small neighborhood around $\theta$ in the relative interior of $\mathtt{U}$. By this construction, the dimension of $\mathtt{U}$ captures the number of directions the learner needs to explore beyond those directions which the optimal policy does not explore. The second part of the condition, that $[\dop_\theta \VEC  K(\theta)] \tilde \theta \neq 0$, pertains to the change of variables $\theta \mapsto K(\theta)$ and relates to the fact that the learner must not necessarily be able to identify $\theta$ from an optimally regulated trajectory, but rather $K(\theta)$.

The second condition, presented below, is that having bounded regret growth, say on the order $\sqrt{T}$, effectively constrains the experiments available to the learner on the subspace the optimal policy does not explore. In other words, the condition formalizes the exploration-exploitation trade-off in LQG in terms of a regret constraint on Fisher information. This is reflected in the proof strategy we pursue in the sequel. Namely, we restrict attention to those policies which attain low regret, $O(\sqrt{T})$. However, these policies necessarily generate experiments with relatively little information content, which in turn implies that the regret of these policies cannot be too small, that is at least $\Omega(\sqrt{T})$.

\begin{definition}
\label{infobounddef}
Fix an $(\mathtt{U},\e)$-locally uninformative instance  $(\theta, A(\cdot),B(\cdot),C(\cdot), Q,R,\Sigma_W,\Sigma_V)$ and a constant $L>0$. We say that the instance is $(\mathtt{U},L)$-information-regret-bounded if for any policy $\pi$, for all $T \in \mathbb{N}$, for all $\tilde \theta \in  B(\theta,\e)$ and any matrix $V_0$ with orthonormal columns spanning $\mathtt{U}$:
\begin{align}
\label{informationcomparsoneq}
\tr V_0^\T \I( \pi;\tilde\theta,T) V_0\leq L \mathsf{R}_T^\pi(\theta).
\end{align}
\end{definition}
Roughly speaking, Definition~\ref{infobounddef} asks that $\dim \mathtt{U}$-many eigenvalues of the information matrix, pertaining to a particular policy $\pi$, satisfy a perturbation bound with respect to the regret of that same policy, $\pi$. In particular, if the condition holds, any policy with $O(\sqrt{T})$ regret will yield an information matrix of which the smallest eigenvalue is also $O(\sqrt{T})$. This should be contrasted with the typical parametric iid design scenario, in which the information matrix  scales linearly with the samples.

The conditions given in \Cref{uninfdef} and \Cref{infobounddef} reveal the key elements needed to prove a regret lower bound on the order of magnitude $\sqrt{T}$, as is done in Theorem~\ref{thm:sfregretlb}. However, the question remains as to which systems these conditions actually apply. To this end, we spend the remainder of this section demonstrating that the conditions given in definitions~\ref{uninfdef} and \ref{infobounddef} are far from vacuous. We prove in Section~\ref{informationstatebounds} that a large class of state feedback systems satisfy both \Cref{uninfdef} and \Cref{infobounddef}.  For instance Lemma~\ref{statefeedbackinfocomp} together with Proposition~\ref{simchoprop} proves that almost any state feedback system with both $A$ and $B$ completely unknown satisfies these conditions. The corresponding nullspace is rather more difficult to characterize for partially observed systems, and we postpone our discussion of these to \Cref{ch:regretpo}. However, at a high level the proof approach detailed below is still valid: one needs to characterize the parameter directions that the optimal policy does not persistently excite and then show that these unexplored directions are necessary for correctly identifying the optimal policy.\footnote{The notion of under-explored parameter direction becomes more subtle for partially observed systems for two reasons: non-uniqueness of realization and the fact that the policy itself has an internal state.}

\subsection{Low Regret Experiments}
\label{informationstatebounds}
We now initiate our study of what we informally refer to as low regret experiments. As mentioned above, the main idea is that if the optimal policy $K(\theta)$ does not provide sufficient exploration, its application to the system yields a degenerate information matrix. The next step is to note that any controller with bounded regret---which can be thought of as the norm of a particular controller to the optimal controller---cannot yield  a particularly good experiment either. While such an experiment is not necessarily singular, it should at least be ill-conditioned. It will be convenient to first calculate the Fisher information of our "experiments".

\begin{lemma}
Fix $T\in \mathbb{N}$. Suppose $C=I_{\dx}$ and $V_t =0 $ for all $t$. The Fisher information under any policy $\pi$ is given by:
\begin{align}
\label{statefeedbackinfo}
\I(\pi;\theta,T)  =\E^\pi_\theta \sum_{t=0}^{T-1}  [\dop_\theta[A(\theta)X_t+B(\theta)U_t]] ^\T \Sigma_W^{-1} \dop_\theta[A(\theta)X_t+B(\theta)U_t].
\end{align}
\end{lemma}

\begin{proof}
Follows immediately by the chain rule for Fisher information \eqref{eq:fichain} and the conditional dependence structure
\begin{align*}
X_{t+1} | (\mathtt{AUX},X_0,\dots,X_t)& \sim {N}(A(\theta) X_{t}+B(\theta)U_t,\Sigma_W),
\end{align*}
the chain rule (which applies due to \Cref{lem:fisherreg}) and \Cref{lem:Gaussianfisher}.
\end{proof}

To make the dependence on $[A(\theta) \: B(\theta)$] more explicit, we may  rewrite (\ref{statefeedbackinfo}) by vectorizing:
\begin{align}
\label{statefeedbackinfo2}
{\I}(\theta;\pi,T)&=\E^\pi_\theta \sum_{t=0}^{T-1}  [\dop_\theta \VEC [A(\theta) \: B(\theta)]] ^\T [Z_t Z_t^\T \otimes \Sigma_W^{-1}] \dop_\theta \VEC [A(\theta)\: B(\theta)]
\end{align}
where $Z_t^\T = \begin{bmatrix} X_t^\T & U_t^\T \end{bmatrix}$.

Note that for the simple parametrization $\VEC [A(\theta) \: B(\theta)]=\theta$,  the Jacobian $  [\dop_\theta \VEC [A(\theta) \: B(\theta)]$ is equal to the identity matrix $I_{d_\Theta}$. In this case, (\ref{statefeedbackinfo2}) is proportional to the covariates matrix used in the ``denominator'' of the least squares estimator. This is satisfying, as this means that our results imply that if the least squares estimator becomes ill-conditioned, there is little else that can be done. A more direct consequence of the representation (\ref{statefeedbackinfo2}) is that an algebraic condition for uninformativeness is straightforward to derive.

\begin{proposition}
\label{prop:charpropsf}
The instance  $(\theta, A(\cdot), B(\cdot), Q,R,\Sigma_W)$ is $(\mathtt{U},\e)$-locally uninformative if and only if for every $v \in \mathtt{U}\setminus\{0\}$ and every $\tilde \theta \in B(\theta,\e)$
\begin{equation}
\begin{aligned}
\label{ch6:uninformativechar}
 v \in& \ker \: [\dop_\theta \VEC [A(\tilde \theta)\: B(\tilde \theta)]]^\T \left[ H(\theta)H^\T(\theta)  \otimes  \Sigma_W^{-1} \right]\dop_\theta \VEC [A(\tilde \theta)\: B(\tilde \theta)]\\
 v\notin& \ker \dop_\theta \VEC K(\theta)
\end{aligned}
\end{equation}
where
\begin{align}\label{eq:hisasin}
H(\theta) = \begin{bmatrix}I_{\dx} \\ K(\theta) \end{bmatrix}.
\end{align}
\end{proposition}

\begin{proof}
Write, for each $t$ and $\tilde \theta\in B(\theta,\e)\cap \mathtt{U}$, 

\begin{align*}
&[\dop_\theta[[A(\tilde \theta) \: B(\tilde \theta)]Z_t]]^\T  \Sigma_W^{-1} \dop_\theta[[A(\tilde \theta) \: B(\tilde \theta)]Z_t] \\
&= [ (Z_t^\T \otimes I_{\dx})\dop_{\theta} \VEC [A(\tilde \theta) \: B(\tilde \theta)]]^\T  (Z_t^\T \otimes \Sigma_W^{-1})\dop_{\theta} \VEC [A(\tilde \theta) \: B(\tilde \theta)]\\
 &=[ \dop_{\theta} \VEC [A(\tilde \theta) \: B(\tilde \theta)] ]^\T (Z_t \otimes I_{\dx})  (Z_t^\T \otimes \Sigma_W^{-1})\dop_{\theta} \VEC [A(\tilde \theta) \: B(\tilde \theta)]\\
 &=[ \dop_{\theta} \VEC [A(\tilde \theta) \: B(\tilde \theta)] ]^\T [Z_tZ_t^\T \otimes  \Sigma_W^{-1} ]\dop_{\theta} \VEC [A(\tilde \theta) \: B(\tilde \theta)],
 \end{align*}
where $\dz = \dx + \du$. Notice now that taking expectation under $(\pi_\star(\theta),\tilde \theta)$:
\begin{align*}
\E^{\pi_\star}_{\tilde \theta} Z_t Z_t^\T &=\E^{\pi_\star}_{\theta} \begin{bmatrix} X_t \\ K(\theta)X_t \end{bmatrix}\begin{bmatrix} X_t^\T & (K(\theta)X_t)^\T \end{bmatrix} =\begin{bmatrix}I_{\dx} \\ K(\theta) \end{bmatrix}\E^{\pi_\star}_{\tilde \theta} [X_tX_t^\T]\begin{bmatrix} I_{\dx} & K^\T(\theta) \end{bmatrix}.
\end{align*}
Since $\E^{\pi_\star}_{\tilde \theta} X_t X_t^\T \succeq \Sigma_W\succ 0$, this has the same nullspace as
\begin{align}
\label{lsepremult}
 \begin{bmatrix}I_{\dx} \\ K(\theta) \end{bmatrix}\begin{bmatrix} I_{\dx} & K^\T(\theta) \end{bmatrix} = H(\theta) H^\T(\theta)
\end{align}
and the result is established.
\end{proof}

In Proposition~\ref{simchoprop} we specialize the above result to the parametrization  $ \VEC  [A(\theta) \: B(\theta)] = \theta$.

\paragraph{Information Comparison} We next turn our attention to establishing that uninformative state feedback systems satisfy the information-regret-boundedness property.

\begin{lemma}
\label{statefeedbackinfocomp}
Suppose $C=I_{\dx}$ and $V_t =0 $ for all $t$ and fix a $(\mathtt{U},\e)$-locally uninformative instance $(\theta, A(\cdot), B(\cdot), Q,R,\Sigma_W)$. For every $\theta' \in B(\theta,\e)$ and every $T\in\N$:
\begin{align*}
\tr V_0^\T \I(\pi;\theta',T) V_0 \leq \tr ( \Sigma_W^{-1}) \left( \sup_{\bar\theta\in B(\theta,\e)} \| \dop_\theta[A(\bar \theta) \:  B(\bar \theta)] \|_\infty^2\right) \|(B^\T P(\theta)B+R)^{-1}\|_{\infty} \mathsf{R}_T^\pi(\theta)
\end{align*}
where as before, the columns of $V_0$ span $\mathtt{U}$.
\end{lemma}

In other words, regret bounds information and uninformative state feedback systems are $L$-information-regret-bounded with $L=\sup_{\bar \theta\in B(\theta,\e)}\| \dop_\theta[A(\bar \theta) \:  B(\bar \theta)] \|_\infty^2\|(B^\T P(\theta)B+R)^{-1}\|_{\infty} \tr ( \Sigma_{w}^{-1}) $.

While the full proof of Lemma~\ref{statefeedbackinfocomp} can be found in \Cref{sec:proof:statefeedbackinfocomp}, the intuition for the proof below is as follows. Both the regret $\mathsf{R}_T^\pi(\theta)$ and the Fisher information $\I(\pi;\theta,T)$ are expectations of quadratic forms in the variables $X_t,U_t, t=0,\dots,T-1$. Knowing that the optimal policy $\pi_\star$ renders the $\I(\pi_\star;\theta,T)$ singular, we can control the small eigenvalues of $\I(\pi;\theta,T)$ in terms of the regret of that policy by a simple Taylor approximation. In other words, we regard the regret of a policy simply as a measure of its deviation from $\pi_\star$ and incorporate this as a constraint on experiment design.

\paragraph{Unstructured Uncertainty} In the literature, much  attention has been given to the case in which both $A$ and $B$ are completely unknown. This corresponds to the parametrization $\VEC [A(\theta) \: B(\theta) ] = \theta$. Our characterization of the nullspace of the Fisher information  takes a particularly simple form for this parametrization. 

\begin{proposition}
\label{simchoprop}
Suppose that $\VEC \begin{bmatrix} A(\theta) & B(\theta) \end{bmatrix} = \theta$. Suppose further that $\det (A+BK) \neq 0$. Then the information singular subspace $\mathtt{U}$ is unique, is equal to $\ker HH^\T \otimes \Sigma_W^{-1}$, and has dimension  
\begin{align*}
\dim \mathtt{U} = \dx \du
\end{align*}
where $H$ is as in \eqref{eq:hisasin}.
\end{proposition}

\begin{remark}
If the system realization $(A,B,\sqrt{Q})$ is minimal and $B$ has full column rank then by Lemma 3.4 in \cite{polderman1986necessity} $\det (A+BK) \neq 0$ is equivalent to $\det A \neq 0$.
\end{remark}

\begin{proof}
 Since $HH^\T \in \mathbb{R}^{(\dx + \du )\times (\dx + \du)}$ is the outer product of two tall matrices, with an identity of size $\dx$ in the first, top-left, block, we have $\dim \ker HH^\T = \du$. Moreover, 
 \begin{align*}
 \ker HH^\T = \{ (x,u ) \in \mathbb{R}^{\dx +\du} : x = - K^\T u \}
 \end{align*}
 so that for any $w\in \mathbb{R}^{\dx}$, any vector, $\theta$ of the form ($u\in \R^{\du}$)
 \begin{align}
 \label{ch6:thetaform}
\theta = \begin{bmatrix}
-K^\T u\\
u
\end{bmatrix}
\otimes w = \begin{bmatrix}
-(K^\T u) \otimes w\\
u \otimes w
\end{bmatrix}
=
\begin{bmatrix}
-(K^\T \otimes I_{\dx}) u \otimes w\\
u \otimes w
\end{bmatrix}
\end{align}
satisfies $\theta \in \ker HH^\T \otimes \Sigma_W^{-1}$. We note that the dimension of the span of such $\theta$ is $\dx \du$, since there are no constraints on the choice of $u$ and $w$. Moreover, all such $\theta$ satisfying (\ref{ch6:thetaform}) can be obtained as the image of the composition of the vectorization operator with $ \Delta \mapsto \begin{bmatrix} -\Delta K & \Delta \end{bmatrix} \in \mathbb{R}^{\dx\times (\dx +  \du)}$. To see this, write
\begin{align}
\label{ch6:simchoform}
\VEC \begin{bmatrix} -\Delta K & \Delta \end{bmatrix} = \begin{bmatrix}
-\VEC \Delta K \\
\VEC \Delta 
\end{bmatrix}
=\begin{bmatrix}
-(K^\T \otimes I) \VEC \Delta\\
\VEC \Delta
\end{bmatrix}
\end{align}
so that identification follows by setting $\VEC \Delta = u\otimes w$ in (\ref{ch6:thetaform}).

We recall Lemma~2.1 from \cite{simchowitz2020naive} (see also \cite{abeille2018improved}) which establishes that
\begin{align}
\label{ch6:simchoder}
 \frac{d}{da} K(A-a\Delta K(A,B), B+\Delta)\Big|_{a=0} = -(R+B^\T P B)^{-1} \Delta^\T P (A+BK)
\end{align}
where we have allowed ourselves some abuse of notation in the obvious identification of $K(A,B) = K(\theta)$ to ease the translation from \cite{simchowitz2020naive}. Now, what is important is that, provided that $A+BK$ is nonsingular (\ref{ch6:simchoder}) is non-zero for all non-zero $\Delta \in \mathbb{R}^{\dx\times \du}$, $\dop_\theta \VEC K(\cdot)$ has non-zero action on $\theta$ as in  (\ref{ch6:thetaform}). Hence combining (\ref{ch6:thetaform}) and (\ref{ch6:simchoform}) with (\ref{ch6:simchoder}) implies that $\dim \mathtt{U} \geq \dx \du$. However, this is maximal since the rank of $HH^\T \otimes \Sigma_W^{-1}$ is $\dx^2$. Hence $\dim \mathtt{U} = \dx \du$ and the subspace $\mathtt{U}$ is in fact unique (it consists of the entire kernel of the Fisher information).
\end{proof}

In other words, what we have shown is the orthogonality of the two nullspaces defined in \Cref{prop:charpropsf} when specialized to the LQR setup with unknown and unstructured $A$ and $B$ matrix. It is interesting to note that we arrive at the variation of parameters $\begin{bmatrix} A-\Delta K & B+\Delta \end{bmatrix}$ after the change of coordinates (\ref{ch6:thetaform}-\ref{ch6:simchoform}) as a consequence of checking paramter variations that lead to a degenerate information matrix, whereas  \cite{simchowitz2020naive} arrives at the same variation by directly considering variations which generate indistinguishable trajectories. Of course, these perspectives are nearly equivalent, as the singularity of Fisher information implies the (local) indistinguishability of the distributions of the trajectories. 

Moreover, Proposition~\ref{simchoprop} is also related to a much earlier observation of \citet{polderman1986necessity}. He established the neccesity of identiying the true parameter $\theta$ to identify $K(\theta)$ which is mirrored in our result. We show that no elements in the nullspace of Fisher information at $\theta$ are in the nullspace of the derivative of $K(\theta)$. In other words, small variations in the parameter space with singular information under the optimal policy yield small variations in optimal policy. Proposition~\ref{simchoprop} can thus be seen as the local analague of Polderman's identifiability result.

%% file: sections/statefeedback.tex
\section{Regret Lower Bounds for State Feedback Systems} \label{sec:sf}

\begin{theorem}
\label{thm:sfregretlb}
Fix $d\leq \dim \mathtt{U}$, fix $\e>0$ and assume that the system $(\theta, A(\cdot),B(\cdot), Q,R,\Sigma_W)$ is $\e$-locally uninformative and $(\mathtt{U},L)$-information-regret-bounded. Suppose A1 and A2 apply for every $\theta'\in B(\theta,\e)$ and let $\Psi$ satisfy:
\begin{align*}
    \Psi \preceq  \frac{1}{2}\sum_{j=0}^{T^{1/16}}(A(\theta')+B(\theta')K(\theta'))^j \Sigma_W  ((A(\theta')+B(\theta')K(\theta'))^\T)^j.
\end{align*}
for all $\theta' \in B(\theta,\e)$, and let $N_{\star}$ be any matrix satisfying $N_{\star} \preceq (B^\T(\theta') P(\theta')B(\theta')+R)$ for all $\theta' \in B(\theta,\e)$.

Fix now $d \in \{1,\dots,\dim \mathtt{U}\}$. There exists an orthonormal matrix $W_0$ with $\spn W_0 \subset \mathtt{U}$ such that for any smooth, compactly supported prior $\mu$ on $\{\theta+W_0 v: \| v\| \leq \e \}$ we have that:
\begin{equation}
    \begin{aligned}\label{eq:sfregretlb}
    & \sup_{\theta' \in B(\theta,\e)}\mathsf{R}_T^\pi(\theta')\\
    &\geq  \sqrt{T} \sqrt{\frac{1+\dim \mathtt{U}-d}{8L}}\sqrt{ \inf_{\theta',\tilde\theta \in B(\theta,\e)}\tr \Bigg((\Psi\otimes N_\star)(\dop_\theta \VEC(K(\theta'))
 W_0W_0^\T (\dop_\theta \VEC(K(\theta'))^\T\Bigg)}.
    \end{aligned}
\end{equation}
as long as $\sqrt{T} \geq\frac{(1+\dim \mathtt{U}-d)\lambda_{\max}(\J(\mu))}{LC}$ and 
\begin{align}\label{eq:burninsfregretlb}
    T\geq \sup_{\theta'\in B(\theta,\e)}\mathsf{poly}\left(9^{\dx},\opnorm{B},\opnorm{A},\opnorm{Q^{-1}}\opnorm{R},\opnorm{R^{-1}},\opnorm{B},\opnorm{P},\opnorm{K},\opnorm{\Sigma_W},C \right)
\end{align}
for a universal polynomial function $\mathsf{poly}$ and where:
\begin{align*}
    C \triangleq   \sqrt{\frac{(1+\dim \mathtt{U}-d)}{ 8 L C }}\sqrt{ \inf_{\theta',\tilde\theta \in B(\theta,\e)}\tr \Bigg((\Psi\otimes N_\star)(\dop_\theta \VEC(K(\theta'))
 W_0W_0^\T (\dop_\theta \VEC(K(\theta'))^\T\Bigg)}.
\end{align*}

\end{theorem}

The result can be significantly simplified if stated asymptotically.

\begin{corollary}
Assume that A1 and A2 apply to the system $(\theta, A(\cdot),B(\cdot), Q,R,\Sigma_W)$. We further assume that this instance is $\e$-locally uninformative for some $\e>0$ and $(\mathtt{U},L)$-information-regret-bounded. Let $d \in \{1,\dots,\dim \mathtt{U}\}$. There exists a matrix $W_0$ with $d$ orthonormal columns which all lie in $\mathtt{U}$ such that for any admissible policy $\pi$ and any $\alpha\in (0,1/4)$:
\begin{equation}
    \begin{aligned}\label{eq:asymptoticlb}
    & \liminf_{T\to \infty} \sup_{\theta' \in B(\theta,T^{-\alpha})}\frac{\mathsf{R}_T^\pi(\theta')}{\sqrt{T}}\\
    &\geq   \sqrt{\frac{1+\dim \mathtt{U}-d}{8L}}\sqrt{ \tr \Bigg((\Sigma^\star_X(\theta)\otimes (B^\T(\theta)P(\theta)B(\theta)+R))(\dop_\theta \VEC(K(\theta))
 W_0W_0^\T (\dop_\theta \VEC(K(\theta))^\T\Bigg)}
    \end{aligned}
\end{equation}
where
\begin{align}\label{eq:sigmastardef}
    \Sigma^\star_X(\theta) \triangleq \lim_{T\to\infty} \frac{1}{T}\sum_{t=0}^{T-1}\E_{\theta}^{\pi_\star(\theta)} \left[X_tX_t^\T \right].
\end{align}
\end{corollary}
The proof is a straightforward application of \Cref{thm:sfregretlb} invoking the continuity of the problem parameters. The only critical detail is to notice that by choosing $\e = T^{-\alpha}$ with $\alpha \in (0,1/4)$ we may choose the prior $\mu$ such $\tr \J(\mu) \lesssim T^{2\alpha}$. Hence for such a choice of $\mu$ the burn-in pertaining to $\mu$ in \Cref{thm:sfregretlb}, $\sqrt{T} \gtrsim \tr \J(\mu)$ is satisfied asymptotically. We now comment on the hardness result \eqref{eq:asymptoticlb}.
\begin{itemize}
    \item First and foremost, logarithmic regret is impossible under the hypotheses: the conjunction of uninformativeness and information-regret-boundedness implies that there is insufficient curvature near the optimal instance making significant exploration necessary. This leads to regret on the order of magnitude $\sqrt{T}$.
       \item By construction, the columns of $W_0$ are orthogonal to the nullspace of $\dop_\theta \VEC(K(\theta))$. Hence the trace appearing in the square root in \eqref{eq:asymptoticlb} is nonzero (with dimensional dependence $d$).
    \item There is considerable flexibility in the parametrization. This allows us to "concentrate" the lower bound on particular system parameters which we might expect to be particularly hard to learn. We will pursue this theme further in the next section.
    \item The lower bound is proportional to $\sqrt{\frac{\dim \mathtt{U}}{L}}$. This factor captures the tension between exploration and exploitation. As noted before $\dim \mathtt{U}$ captures the number of directions not excited by the optimal policy while the term $L$ captures the sensitivity of the cost to exploration in these directions.
    \item The factors $\Sigma_X^\star$ and $P$ can be thought of as to capture the control-theoretic hardness of the particular instance under consideration. These will be large if the optimal policy operates near marginal stability.
    \item The term $\dop_\theta \VEC(K(\theta))$ does not have any intrinsic interpretation. It is simply a jacobian term arising in the implicit change of variables appearing as a consequence of our choice of parameter geometry before applying \Cref{thm:vtineq}. Put differently; this term arises since we seek to estimate $K(\theta)$ and not $A(\theta)$ or $B(\theta)$ which are the matrices that we have parametrized.
\end{itemize}

\paragraph{Improving the Lower bound of \cite{simchowitz2020naive}}

In the setting in which both $A$ and $B$ are completely unknown $\VEC \begin{bmatrix} A(\theta) & B(\theta) \end{bmatrix} = \theta$, we have the following result, recovering an earlier result of \cite{simchowitz2020naive} but with  improved constants.

\begin{corollary}\label{statefeedbackcorr}
Assume that A1 and A3 hold for the fixed tuple $(A,B)$ with corresponding optimal policy $K$, Riccati matrix $P$ and state covariance matrix $\Sigma^\star_X$ (recall \eqref{eq:sigmastardef}).Then for every $\alpha \in (0,1/4)$ and any admissible policy $\pi$:
\begin{equation}\label{eq:simcorr1}
    \liminf_{T\to \infty} \sup_{ \substack{A' , B': \\ \| [A'-A \: B'-B] \|_{\mathsf{op}} \leq T^{-\alpha}}} \frac{ \mathsf{R}_T^\pi(A',B')}{\sqrt{T}}
\geq c \sqrt{
 \frac{ \du  \lambda_{\min}(\Sigma_W) \tr \Bigg((P[\Sigma^\star_X-I_{\dx}]P\Bigg)}{\dx(1+ \opnorm{KK^\T})}}
\end{equation}
for some universal positive constant $c>0$. Under the same hypotheses we also have that any admissible policy $\pi$ satifies:
\begin{equation}\label{eq:simcorr2}
    \liminf_{T\to \infty} \sup_{ \substack{A' , B': \\ \| [A'-A \: B'-B] \|_{\mathsf{op}} \leq T^{-\alpha}}} \frac{ \mathsf{R}_T^\pi(A',B')}{\sqrt{T}}\\
\geq c' \sqrt{\dx \du^2} \sqrt{
 \frac{  \lambda_{\min}(\Sigma_W) \lambda_{ \min}\left(P[\Sigma^\star_X-I_{\dx}]P\right) }{1+ \opnorm{KK^\T}}}.
\end{equation}
for some second universal positive constant $c'>0$.
\end{corollary}

The dimensional dependency in \Cref{eq:simcorr2} is optimal \citep{simchowitz2020naive}. However, comparing with \cite{simchowitz2020naive} our lower bound exhibits improved scaling in system-theoretic constants. This is especially true of \Cref{eq:simcorr1}. Their bound scales as $1/\opnorm{P^2(\theta)}$ which becomes small if the optimal controller is operates near marginal stability. By contrast, our bound captures the intuition that systems that are hard to optimally regulate are also hard to learn optimally regulate. To appreciate this contrast, consider for instance an open loop unstable scalar system:
\begin{align}\label{eq:scalarsystemtobeplotted}
    X_{t+1}=aX_t + bU_t +W_t
\end{align}
with $|a|\geq 1$. If we instantiate our lower bound and let $|b|\to 0$, the right hand side of \Cref{eq:simcorr1} tends to positive infinity (to see this, use $K^\T R K \preceq P$). By contrast, the corresponding lower bound in \cite[Theorem 1]{simchowitz2020naive}---which is proportional to $1/\opnorm{P}^2$---tends to 0---see also \Cref{fig:thefig}.

\begin{figure}[h]
    \centering
    \includegraphics[scale=0.8]{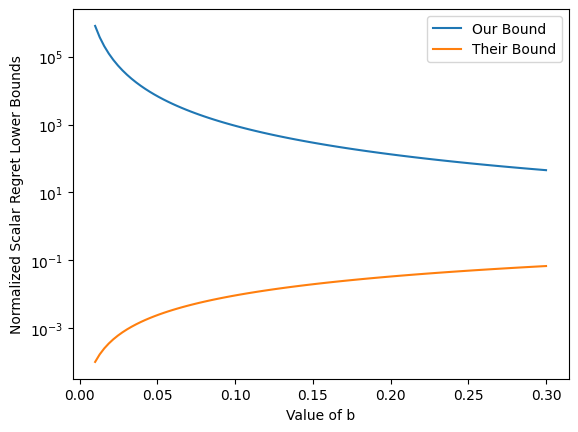}
    \caption{Above, we plot the lower bound of \cite{simchowitz2020naive} and compare it to ours for the scalar system in \eqref{eq:scalarsystemtobeplotted} with $a=1$ and $b$ varying. Noise variance, $Q$ and $R$ are all chosen to be unity. We have omitted the time dependency---which is the same for both bounds---and chosen the larger of the two values appearing in the minimum in the bound of \cite{simchowitz2020naive}. Our lower bound is stronger in this regime, and in fact diverges as $b\to 0$, whereas theirs tends to $0$. In other words, our bound reflects the fact that learning to control becomes harder as control authority is lost.}
    \label{fig:thefig}
\end{figure}

%% file: sections/partiallyobs.tex
\section{Extension to Partially Observed Systems}\label{sec:po}

\label{ch:regretpo}
In this section we seek to understand the hardness of adaptive LQR in the partially observed setting \eqref{eq:lds_ac}-\eqref{eq:plds_ac}. We will see that there are two new failure modes that arise due to poor observability of the inputs and due to poor observability of the state.

\subsection{An Easy-to-Analyze Family of Systems}
It is more difficult to analyze the reduction of \Cref{lem:regretlbrelaxed} when the optimal filter is nontrivial. Here, we will circumvent this issue by considering a particular class of systems in which the effect of filtering can be separated from that of control in the context of regret minimization. Let $\dx$ be divisible by $3$ and consider:
\begin{equation}
    \begin{aligned}
    A(\theta) = \begin{bmatrix} A_{11} &0 &0\\
    I_{\dx/3}& 0 &0\\
    0&0&0
    \end{bmatrix} \qquad 
    B(\theta) =
    \begin{bmatrix}
    0_{\dx/3\times \du}\\ B_2(\theta) \\ B_2(\theta)
    \end{bmatrix}\qquad
    C(\theta) = \begin{bmatrix} C_{11} & 0 & 0\\ 0 & 0 & C_{23} \end{bmatrix}
    \end{aligned}
\end{equation}
where $A_{11}\in \R^{\dx/3\times\dx/3}$ is a fixed stable matrix, $C_{11}, C_{23}\in \mathbb{R}^{\dx/3 \times \dx/3}$ are also fixed and $0\in \R^{\dx/3 \times \dx/3}$, such that the only uncertainty in the parameters is involved in actuation (via $B_2(\theta)\in \R^{\dx/3\times \du}$). For simplicity fix $\lambda>0$ and set $Q= I_{\dx}$ and $R=\lambda I_{\du}$. To simplify the exposition, we will also set $\theta= \VEC B_2(\theta)$. Let us also assume that $\Sigma_W\succ 0$ and $\Sigma_V\succ 0$ have the following block structure
\begin{align*}
    \Sigma_W &= \begin{bmatrix}\Sigma_{W_1} &0 &0\\ 0 & \Sigma_{W_2} &0 \\ 0 &0 & \Sigma_{W_3}\end{bmatrix},\\
      \Sigma_V &= \begin{bmatrix}\Sigma_{V_1} &0 \\ 0 & \Sigma_{V_2} \end{bmatrix},
\end{align*}
and that as before the sequence $\{W_t\}_{t\in \N}$ and $\{V_t\}_{t\in \N}$ are mutually independent and \iid\ Gaussian.

This gives rise to the following family of linear dynamical systems:
\begin{equation}\label{eq:structuredposystem}
    \begin{aligned}
    \begin{bmatrix}X_{t+1}^1 \\ X_{t+1}^2 \\ X_{t+1}^3\end{bmatrix}& = \begin{bmatrix} A_{11} &0 &0\\
    I_{\dx/3}& 0 &0\\
    0&0&0
    \end{bmatrix} \begin{bmatrix}X_{t}^1 \\ X_{t}^2 \\ X_{t}^3\end{bmatrix}
    +
    \begin{bmatrix}
    0_{\dx/3\times \du}\\ B_2(\theta) \\ B_2(\theta)
    \end{bmatrix}U_t + \begin{bmatrix}
    W_t^1\\W_t^2\\W_t^3
    \end{bmatrix}\\
    \begin{bmatrix}Y_t^1\\Y_t^2 \end{bmatrix}&=
     \begin{bmatrix} C_{11} & 0 & 0\\ 0 & 0 & C_{23}  \end{bmatrix}\begin{bmatrix}X_{t}^1 \\ X_{t}^2 \\ X_{t}^3\end{bmatrix} + \begin{bmatrix} V_t^1 \\ V_t^2 \end{bmatrix}.
    \end{aligned}
\end{equation}
Notice that the first (block-)coordinate of $Y$ is a noisy observation of the first coordinate of the state: $Y_t^1 = C_{11} X_t^1+V_t^1$. The second coordinate of $Y$ is simply a noisy observation of the input action $Y_t^2 = C_{23}B_2(\theta)U_{t-1}+C_{23}W_{t-1}^3+V_t^2$. Let us also note that for A1 to hold---i.e. for $(A(\theta),C(\theta))$ to be detectable---it is sufficient that $(A_{11},C_{11})$ is detectable. Since $A_{11}$ is stable by assumption, A1 is immediate. However, \eqref{eq:structuredposystem} is not observable.

The structure \eqref{eq:structuredposystem} means that the regret has the following simple representation.

\begin{lemma}
Assume that A1-A3 holds. For every instance of the form \eqref{eq:structuredposystem}:
\begin{equation}
    \begin{aligned}
    \mathsf{R}_T^\pi(\theta) = \sum_{t=0}^{T-1}\E\left[ (U_t-K_1(\theta)\E_\theta^\pi[X_t^{1}|\mathcal{Y}_t ])^\T (2B_2^\T(\theta) B_2(\theta) +\lambda I_{\du}) (U_t-K_1(\theta)\E_\theta^\pi[X_t^{1}|\mathcal{Y}_t ]) \right]
    \end{aligned}
\end{equation}
where
\begin{align}\label{eq:k1inpobs}
    K_1(\theta) =-(2 B_2^\T(\theta) B_2(\theta)+\lambda I_{\du})^{-1}B_2^\T(\theta).
\end{align}
\end{lemma}

\begin{proof}
Let us write out the cost:
\begin{align*}
    \mathsf{V}_T^\pi(\theta) &= \E_\theta^\pi\sum_{t=0}^{T-1}   \|X_{t}^1\|_2^2+ \|X_{t}^2\|_2^2+\|X_{t}^3\|_2^2+\lambda \|U_t\|_2^2\\
    &= \E_\theta^\pi\sum_{t=0}^{T-1}   \|X_{t}^1\|_2^2+ \|X_{t-1}^1+B_2(\theta )U_{t-1}+W_{t-1}^2\|_2^2+\|B_2(\theta)U_{t-1}+W_{t-1}^3\|_2^2+\lambda \|U_t\|_2^2.
\end{align*}
Since $\E_\theta^\pi\|X_t^1\|_2^2$ is independent of $\pi$, it follows immediately  that the optimal state feedback policy is structured as:
\begin{align*}
    K(\theta) =\begin{bmatrix}
        K_1(\theta) & 0 & 0
    \end{bmatrix}
\end{align*}
where it can readily by computed that $K_1$ is as in \eqref{eq:k1inpobs}. The result follows by invoking \Cref{regretlemma}. 
\end{proof}

The next step is to notice that $\E_\theta^\pi[X_t^1|\mathcal{Y}_t]$ is independent of both both $\theta$ and $\pi$. To see this, notice that as long as A1 and A3 hold, the evolution of $\E_\theta^\pi[X_t^1|\mathcal{Y}_t]=\hat X_t^1$ is given by the following filtering equation
\begin{align}
\label{eq:stateestsillypobs}
\hat X_{t+1}^1& = A_{11}\hat X_t +  F_1 [Y_{t+1}^1-C_{11}A_{11}\hat X_t^1)]
\end{align}
where we recall that filter gain $F_1$ is characterized by
\begin{align}
\label{eq:sDAREsillypobs}
S_{1} &= A_{11}S_{1}A_{11}^\T-A_{11}S_{1}C_{11}^\T(C_{11}S_1C_{11}^\T + \Sigma_{V_1})^{-1}C_{11}S_1A_{11}^\T+\Sigma_{W_1}\\
\label{eq:lDAREsillypobs}
F_1 &= S_1 C_{11}^\T(C_{11} S_{1}C_{11}^\T+\Sigma_{V_1})^{-1}.
\end{align}
All the quantities appearing in \eqref{eq:stateestsillypobs}, \eqref{eq:sDAREsillypobs} and \eqref{eq:lDAREsillypobs} are independent of both $\theta$ and $\pi$. Invoking \Cref{lem:regretlbrelaxed} we obtain the following relaxation.
\begin{lemma}
\label{lem:regretllbrelaxededPO}
Fix $\theta\in \mathbb{R}^{d_\Theta}$ and $\e>0$. Assume that A1-A3 hold for every $\theta \in B(\theta,\e)$. Fix a function $\tau : \mathbb{Z} \to \mathbb{Z}$ with $\tau(t) \geq t$ and a smoothly compactly supported prior $\mu \in C^\infty_c[B(\theta,\e)]$. Then for every instance of form \eqref{eq:structuredposystem}:
\begin{multline}
\label{eq:regretlbrelaxedPO}
 \sup_{\theta'\in B(\theta,\e)}\mathsf{R}_T^\pi(\theta')\\
    \geq\sum_{t=0}^{T-1} \E_{\Theta \sim \mu}\E_{\Theta}^\pi \tr \Big[N_\star  \left( \E[K_1(\Theta)|\mathcal{Y}_{\tau(t)}]-K_1(\Theta)\right) \left(\hat X_t^1 (\hat X_t^1)^\T \right) \left( \E[K_1(\Theta)|\mathcal{Y}_{\tau(t)}]-K_1(\Theta) \right)^\T \Big].
\end{multline}
for any matrix  $N_\star$ satisfying $N_\star\preceq 2B_2^\T(\theta')B_2(\theta') +\lambda I_{\du}$ for every $\theta' \in B(\theta,\e)$, where $K_1$ is as in \eqref{eq:k1inpobs} and where $\hat X_{t+1}^1$ is given by \eqref{eq:stateestsillypobs}.
\end{lemma}


\subsection{Fisher Information Bounds}

The program below parallels the development for the state feedback setting. That is, we seek to characterize the information matrix for policies with low regret ($O(\sqrt{T})$). We begin by providing an expression analogous to \eqref{statefeedbackinfo} for the system \eqref{eq:structuredposystem}.

\begin{lemma}
Fix $T\in \mathbb{N}$ and consider the system \eqref{eq:structuredposystem}. The Fisher information under any policy $\pi$ is given by:
\begin{align}
\label{eq:structuredposysteminfo}
\I(\pi;\theta,T)  =\E^\pi_\theta \sum_{t=0}^{T-1}  [\dop_\theta[C_{23}B_2(\theta)U_t]] ^\T (C_{23}\Sigma_{W_3}C_{23}^\T+\Sigma_{V_2} )^{-1} \dop_\theta[C_{23}B_2(\theta)U_t].
\end{align}
\end{lemma}

\begin{proof}
Note that the only part of the observation $Y$ that depends on $\theta$ is $Y^2$. The result now follows by the chain rule for Fisher information \eqref{eq:fichain} and the conditional dependence structure
\begin{align*}
Y_{t}^2 | (\mathtt{AUX},Y_0,\dots,Y_{t-1})& \sim \mathsf{N}(C_{23}B(\theta)U_t,C_{23}\Sigma_{W_3}C_{23}^\T+\Sigma_{V_2})
\end{align*}
and \Cref{lem:Gaussianfisher}.
\end{proof}

Equipped with \Cref{eq:structuredposysteminfo}, uninformativeness (recall \Cref{uninfdef}) is readily characterized as follows below.

\begin{proposition}
\label{prop:charproppo}
For any $\theta$, the instance \eqref{eq:structuredposystem} is $(\mathtt{U},\e)$-locally uninformative if for every $v \in \mathtt{U}\setminus\{0\}$:
\begin{equation}
\begin{aligned}
\label{eq:uninformativecharpo}
 v \in& \ker   \left[ K_1(\theta)K_1^\T(\theta)  \otimes  C_{23}^\T C_{23} \right]\\
 v\notin& \ker \dop_\theta \VEC K_1(\theta)
\end{aligned}
\end{equation}

\end{proposition}

\begin{proof}
Write, for each $t$ and $\tilde \theta\in B(\theta,\e)\cap \mathtt{U}$, 
\begin{align*}
   & [\dop_\theta[C_{23}B_2(\tilde \theta)U_t]] ^\T (C_{23}\Sigma_{W_3}C_{23}^\T+\Sigma_{V_2} )^{-1} \dop_\theta[C_{23}B_2(\tilde \theta)U_t]\\
   &=[(U_t^\T \otimes C_{23})\dop_\theta \VEC [B_2(\tilde \theta)]] ^\T  [(U_t^\T \otimes (C_{23}\Sigma_{W_3}C_{23}^\T+\Sigma_{V_2} )^{-1} C_{23})\dop_\theta \VEC [B_2(\tilde \theta)]]\\
   &=[\dop_\theta \VEC [B_2(\tilde \theta)]] ^\T  [(U_tU_t^\T \otimes C_{23}^\T(C_{23}\Sigma_{W_3}C_{23}^\T+\Sigma_{V_2} )^{-1} C_{23})\dop_\theta \VEC [B_2(\tilde \theta)]]\\
   &=  [(U_tU_t^\T \otimes C_{23}^\T(C_{23}\Sigma_{W_3}C_{23}^\T+\Sigma_{V_2} )^{-1} C_{23})
\end{align*}
where we used that $\tilde \theta = \VEC B_2(\tilde \theta)$. Notice now that taking expectation under $(\pi_\star(\theta),\tilde \theta)$:
\begin{align*}
\ker \E^{\pi_\star}_{\tilde \theta} U_t U_t^\T \subset \ker  K_{1}(\theta) K_1^\T(\theta).
\end{align*}
Since $C_{23}^\T C_{23}$ has the same nullspace as $C_{23}^\T(C_{23}\Sigma_{W_3}C_{23}^\T+\Sigma_{V_2} )^{-1} C_{23}$, the result follows by \Cref{eq:structuredposysteminfo}.
\end{proof}

Let us now investigate  the second part of condition \eqref{eq:uninformativecharpo}.
\begin{lemma}
\label{lem:kerklemma}
The nullspace $\ker  \dop_\theta \VEC K_1(\theta)$ is equal to the set of solutions $v$ of 
\begin{align}
\label{ch4:goodsamaritan}
v= 2 \Pi ^{-1} \left[ ( K_1^\T B_2^\T \otimes I_{\du})\Pi +(K_1^\T \otimes B_2^\T) \right]v.
\end{align}
where $K_1$ and $B_2$ are evaluated at $\theta$ and where the permutation matrix $\Pi$ maps  $ \VEC  M$ to $\VEC M^\T$.
\end{lemma}

\begin{proof}
We seek solutions $v=\VEC \diff B_2 \in \mathbb{R}^{d_\Theta}$ to 
\begin{align*}
    \dop_\theta \VEC K_1(\theta)  v =0.
\end{align*}
By the identification theorem for matrix derivativs \citep[Theorem 5.11]{magnus2019matrix}, we may equivalently seek solutions to $\diff K_1=0$. We compute the differential of $K_1$ with respect to $B_2=\VEC^{-1}\theta$ as
\begin{multline*}
\diff (2B_2^\T B_2 +\lambda I_{\du})^{-1} B_2^\T\\
= (2B_2^\T B_2+ \lambda I_{\du})^{-1}(\diff B_2)^\T -2 (2B_2^\T B_2+\lambda I_{\du})^{-1}[(\diff B_2)^\T B + F^\T \diff B_2] (2B_2^\T B_2+\lambda I_{\du})^{-1}B_2^\T.
\end{multline*}
Multiplying by $(2B_2^\T B_2+\lambda I_{\du})$, we see that this can be set to zero if and only if there is a root $\diff B_2$ to the equation
\begin{align*}
(\diff B_2)^\T =2 [(\diff B_2)^\T B_2 + B_2^\T \diff B_2] (2B_2^\T B_2+\lambda I_{\du})^{-1}B_2^\T.
\end{align*}
Vectorizing, we find the equivalent equation
\begin{multline*}
\VEC (\diff B_2)^\T =2 \VEC \left( [(\diff B_2)^\T B_2 + B_2^\T \diff B_2] (2B_2^\T B_2+\lambda I_{\du})^{-1}B_2^\T\right)\\
=2( B_2  (2B_2^\T B_2+ \lambda I_{\du})^{-1} B_2^\T \otimes I_{\du}) \VEC \diff B_2^\T +(B_2(2B_2^\T B_2+\lambda I_{\du})^{-1} \otimes B_2^\T) \VEC \diff B_2.
\end{multline*}
Introducing $\Pi$ such that $\VEC (\diff B_2)^\T = \Pi  \VEC \diff B_2$, this is in turn equivalent to the eigenvalue equation
\begin{multline*}
\begin{aligned}
&\VEC \diff B_2 \\
=& 2\Pi^{-1} \left[ ( B_2  (2B_2^\T B_2+ \lambda I_{\du})^{-1} B_2^\T \otimes I_{\du})\Pi  \VEC \diff B_2 +(B_2(2B_2^\T B_2+\lambda I_{\du})^{-1} \otimes B_2^\T) \VEC \diff B_2\right]\\
=& 2 \Pi ^{-1} \left[ ( K_1^\T B_2^\T \otimes I_{\du})\Pi  \VEC \diff B_2 +(K_1^\T \otimes B_2^\T) \VEC \diff B_2\right].
\end{aligned}
\end{multline*}
This completes the proof.
\end{proof}

The next proposition exploits the geometric description of $\ker \dop_\theta \VEC K_1(\theta)$ provided in Lemma~\ref{lem:kerklemma} to give a more direct characterization of uninformativeness.

\begin{proposition}
\label{prop:strucutedpouninf}
Fix $\theta \in \mathbb{R}^{d_\Theta}$, $\e>0$ and suppose that $K_1(\theta)K_1^\T(\theta)$ is singular. Set
\begin{align}\label{eq:pouninfspace}
\mathtt{U} = \{  u \otimes w \in \mathbb{R}^{d_\theta} :  u \in \ker K_1(\theta)K_1^\T(\theta),  w \in \mathbb{R}^{\dx/3} \}.
\end{align}
The instance \eqref{eq:structuredposystem} is $(\mathtt{U},\e)$-uninformative.
\end{proposition}

Note that $\ker K_1(\theta)K_1^\T(\theta)$ is nontrivial as soon as $\du > \dx/3$.

\begin{proof}
We need to prove that no nonzero vector in $\mathtt{U}$ is in $\ker  \dop_\theta \VEC K_1(\theta)$. The idea is to use Lemma~\ref{lem:kerklemma}, and establish that no nonzero vector $v \in \mathtt{U}$  belongs to the eigenspace of the eigenvalue $1$ of $2 \Pi ^{-1} \left[ ( K_1^\T B_2^\T \otimes I_{\du})\Pi +(K_1^\T \otimes B_2^\T) \right]$ appearing in (\ref{ch4:goodsamaritan}).

To that end,  consider any vector of the form $ v =  u \otimes  w $ where $ u \in \ker K_1K_1^\T = \ker K^\T$ and $ w \in \mathbb{R}^{\dx/3}$. Notice that the action of $\Pi$ on vectors $ u \otimes  w$ is described by $\Pi u \otimes  w =  w\otimes u$. We may thus compute:
\begin{equation}
\begin{aligned}
\label{ch4:anothersamaritan}
&2\left[ (K_1^\T B_2^\T \otimes I_{\du})\Pi   +(K_1^\T \otimes B_2^\T) \right] ( u\otimes  w)\\
 &= 2( K_1^\T B_2^\T \otimes I_{\du})\Pi ( u\otimes  w)  + (K_1^\T \otimes B_2^\T) (u\otimes  w) \\
&=2( K_1^\T B_2^\T \otimes I_{\du})(w\otimes  u)  + (K_1^\T \otimes B_2^\T) ( u\otimes  w)\\
&=2( K_1^\T B_2^\T  w\otimes I_{\du} u) + (\underbrace{K_1^\T u}_{=0} \otimes B_2^\T w)\\
&=2( K_1^\T B_2^\T \otimes I_{\du}) ( w \otimes   u)
\end{aligned}
\end{equation}
using $\ker K_1K_1^\T = \ker K_1^\T$.

We shall now prove that the spectrum of $2( K_1^\T B_2^\T \otimes I_{\du}) $ is contained in $(-\infty, 1)$. Since this does not include $1$, it then follows by Lemma~\ref{lem:kerklemma} that all nonzero elements of $\mathtt{U}$ are not in $\ker  \dop_\theta \VEC K_1(\theta)$. To show this, it suffices to prove that  the spectrum of $2( K_1^\T B_2^\T) $ is contained in $(-\infty, 1)$ due to the structure of the eigenvalues for Kroenecker products. 

For the next step, we may assume without loss of generality that $w \in (\ker B_2^\T)^\perp$ by simply decomposing $w$. Notice now that for all such $w \in (\ker B_2^\T)^\perp$ and for $\lambda' \in (0,\lambda)$, we have that
\begin{align}
\label{ch4:monorelation}
 2w^\T K_1^\T B_2^\T w =2w^\T  B_2  (2B_2^\T B_2 +\lambda I_{\du})^{-1} B_2^\T  w < 2w^\T  B_2  (2B_2^\T B_2 +\lambda' I_{\du})^{-1} B_2^\T w.
\end{align}
Moreover, 
\begin{align*}
\lim_{\lambda'\to 0}  2B_2  (2 B_2^\T B +\lambda' I_{\du})^{-1} B_2^\T  w =B_2^\dagger B_2  w =  w
\end{align*}
since $w \in (\ker B_2^\T)^\perp$. Combining this observation with (\ref{ch4:monorelation}) we obtain the inequality
\begin{align*}
 w^\T  (2K_1^\T B_2^\T) w < w^\T w,
\end{align*}
valid for all $w \in (\ker K_1^\T B_2^\T)^\perp$, which can only be true if the largest eigenvalue of $2K_1^\T B_2^\T$ is smaller than $1$ (note that the eigenvalues are nonnegative real since $ K_1^\T B_2^\T \succeq 0$). 
\end{proof}

\paragraph{Information Comparison}

The final preliminary lemma required will be used \Cref{statefeedbackinfocomp}. We again show that regret bounds the small singular values of the Fisher information.

\begin{lemma}
\label{lem:sillypobsinfocomp}
Fix $\theta \in \mathbb{R}^{d_\Theta}$, $\e>0$, consider the instance \eqref{eq:structuredposystem} and suppose that $K_1(\theta)K_1^\T(\theta)$ is singular. Set
\begin{align}
\mathtt{U} = \{  u \otimes w \in \mathbb{R}^{d_\Theta} :  u \in \ker K_1(\theta)K_1^\T(\theta),  w \in \mathbb{R}^{\dx/3} \}.
\end{align}
For every $\theta' \in B(\theta,\e)$ and every $T\in\N$:
\begin{align*}
\tr V_0^\T \I(\pi;\theta',T) V_0 \leq \tr (  C_{23}^\T(C_{23}\Sigma_{W_3}C_{23}^\T+\Sigma_{V_2} )^{-1} C_{23})\opnorm{(2B_2^\T B_2+\lambda I_{\du})^{-1}} \mathsf{R}_T^\pi(\theta)
\end{align*}
where as before, the columns of $V_0$ span $\mathtt{U}$.
\end{lemma}

\subsection{A Hardness Result for LQG}

Equipped with the characterizations of the (restricted) nullspace of \Cref{prop:charproppo} and \Cref{prop:strucutedpouninf} and the information-regret bound of \Cref{lem:sillypobsinfocomp} we are ready to establish the following analogue of \Cref{thm:sfregretlb}.

\begin{theorem}
\label{thm:poregretlb}
Fix $\e>0$ and consider the system \eqref{eq:structuredposystem}. Suppose A1 and A2 apply for every $\theta'\in B(\theta,\e)$. Let $d \in \{1,\dots,\dim \mathtt{U}\}$ where $\mathtt{U}$ is given by \Cref{eq:pouninfspace}. Fix also a smooth, compactly supported prior $\mu$ on $B(\theta,\e)$. 
There exists a polynomial function $\mathsf{poly}$ such that if 
\begin{align*}
      T \geq \mathsf{poly}\left(\dx,\log \left(\sum_{t=0}^\infty \opnorm{A_{11}}^t \right),\log \opnorm{F_1 (C_{11}S_1 C_{11}^\T+ \Sigma_{V_1})F_1^\T  }\right)
\end{align*}
Furtheremore, There exists an orthonormal matrix $W_0$ with $\spn W_0 \subset \mathtt{U}$ such that for any smooth, compactly supported prior $\mu$ on $\{\theta+W_0 v: \| v\| \leq \e \}$ as long as $\sqrt{T} \geq\frac{(1+\dim \mathtt{U}-d)\lambda_{\max}(\J(\mu))}{LC}$  we have that for any admissible policy $\pi$:
\begin{equation}
    \begin{aligned}\label{eq:poregretlb}
    & \sup_{\theta' \in B(\theta,\e)}\mathsf{R}_T^\pi(\theta)\\
    &\geq  \sqrt{T} \sqrt{\frac{1+\dim \mathtt{U}-d}{8L}}\sqrt{ \inf_{\theta',\tilde\theta \in B(\theta,\e)}\tr \Bigg((\Psi\otimes N_\star)(\dop_\theta \VEC(K_1(\theta'))
 W_0W_0^\T (\dop_\theta \VEC(K_1(\tilde\theta))^\T\Bigg)}.
    \end{aligned}
\end{equation}
where $\Psi$ is given by:
\begin{align*}
    \Psi =  \frac{1}{T}\sum_{t=0}^{T-1} \E \hat X_t^1 (\hat X_t^1)^\T- T^{-1/4}\left(\sum_{t=0}^\infty \opnorm{A_{11}}^t \right)^2 \times (I_{\dx/3})
\end{align*}
and where we define
\begin{align*}
L & \triangleq \tr (  C_{23}^\T(C_{23}\Sigma_{W_3}C_{23}^\T+\Sigma_{V_2} )^{-1} C_{23})\opnorm{(2B_2^\T B_2+\lambda I_{\du})^{-1}},\\
    C &\triangleq   \sqrt{\frac{(1+\dim \mathtt{U}-d)}{ 8 L C }}\sqrt{ \inf_{\theta',\tilde\theta \in B(\theta,\e)}\tr \Bigg((\Psi\otimes N_\star)(\dop_\theta \VEC(K_1(\tilde \theta))
 W_0W_0^\T (\dop_\theta \VEC(K_1(\theta'))^\T\Bigg)}
\end{align*}
and $N_{\star}$ is any matrix satisfying  $N_\star\preceq 2B_2^\T(\theta')B_2(\theta') +\lambda I_{\du}$ for all $\theta' \in B(\theta,\e)$.

\end{theorem}

In light of \Cref{lem:regretllbrelaxededPO} and \Cref{lem:sillypobsinfocomp}, the proof is almost identical to that of \Cref{thm:sfregretlb} and is thus omitted. The result can be significantly simplified if stated asymptotically.

\begin{corollary}
Consider the system \eqref{eq:structuredposystem}. Suppose A1 and A2 apply for every $\theta'\in B(\theta,\e)$. Let $d \in \{1,\dots,\dim \mathtt{U}\}$ where $\mathtt{U}$ is given by \Cref{eq:pouninfspace}. Let $L  \triangleq \tr (  C_{23}^\T(C_{23}\Sigma_{W_3}C_{23}^\T+\Sigma_{V_2} )^{-1} C_{23})\opnorm{(2B_2^\T B_2+\lambda I_{\du})^{-1}}$.

There exists a matrix $W_0$ with $d$ orthonormal columns which all lie in $\mathtt{U}$ such that for any admissible policy $\pi$ and any $\alpha \in (0,1/4)$:
\begin{equation}
    \begin{aligned}\label{eq:asymptoticpolb}
    & \liminf_{T\to \infty} \sup_{\theta' \in B(\theta,T^{-\alpha})}\frac{\mathsf{R}_T^\pi(\theta)}{\sqrt{T}}\\
    &\geq   \sqrt{\frac{1+\dim \mathtt{U}-d}{8L}}\sqrt{ \tr \Bigg((\Sigma_{\hat X^1}\otimes (2B_2^\T(\theta)B_2(\theta)+R))(\dop_\theta \VEC(K_1(\theta))
 W_0W_0^\T (\dop_\theta \VEC(K_1(\theta))^\T\Bigg)}
    \end{aligned}
\end{equation}
where $ \displaystyle  \Sigma_{\hat X^1} \triangleq \lim_{T\to\infty} \frac{1}{T}\sum_{t=0}^{T-1} \E \hat X_t^1 (\hat X_t^1)^\T$.
\end{corollary}

While the lower bound \eqref{eq:asymptoticpolb} is interpreted much like its fully observed analogue \eqref{eq:asymptoticlb} (in particular, see the ensuing discussion), there are two new failure modes that arise. First, as $\opnorm{C_{23}}$ tends to zero, the lower bound \eqref{eq:asymptoticpolb} diverges. In this case, the learner faces vanishing information available to them about $B_2$, which is needed to regulate the system. Second, it is no longer the "state" covariance that enters the bound \eqref{eq:asymptoticpolb}, but the \emph{filtered} state covariance $ \Sigma_{\hat X^1} $, which depends both on the stability and detectability of the first mode. In particular, this term does not diverge as stability is lost $\rho(A_{11}) \to 1$, but rather if $\rho(A_{11}) \to 1$ \emph{and} observability is lost (e.g. $\opnorm{C_{11}} \to 0$). In other words, regret minimization becomes hard as observability of unstable modes is lost.

It is also interesting to note that it has been proven by \cite{lale2020logarithmic} that logarithmic regret against the best possible, in hindsight, persistently exciting controller\footnote{We refer to \cite{lale2020logarithmic} for their definition, but this roughly corresponds to having well-conditioned Fisher information.} whenever the covariance of the measurement noise is positive definite, i.e., $\Sigma_V\succ 0$.   Unfortunately, it is not clear whether the optimal LQG controller is persistently exciting under these hypotheses and so the notion of regret in \cite{lale2020logarithmic} may differ from the standard one. With this in mind, we thus also prove a negative result, showing that $\Sigma_V\succ 0$ is not sufficient for logarithmic regret without further assumption. Finally, we point out that the system used in our construction is not minimal, in the sense that \eqref{eq:structuredposystem} is not observable. By contrast, known upper bounds apply to controllable and observable systems \citep{simchowitz2020improper, lale2020logarithmic}. Thus, while our bounds show that logarithmic regret is not \emph{always} possible, even with full rank output noise, it does not rule out this possibility in the exact setting of \cite{lale2020logarithmic}.

%% file: sections/acknowledgements.tex
\section*{Acknowledgements}
This work was completed while Ingvar Ziemann was still at KTH. Ingvar Ziemann is supported by a Swedish Research Council International Postdoc grant.  Henrik Sandberg is supported by the Swedish Research Council (grant 2016-00861). The authors wish to express their gratitude to Bruce Lee, Anastasios Tsiamis, and Yishao Zhou for helpful suggestions and feedback.

%% file: Appendix/proofsSF.tex
\section{Proofs for the State Feedback Setting}

\subsection{Proof of \Cref{thm:sfregretlb}}
Throughout the proof we use $\e= T^{-\alpha}$ and set $T$ sufficiently large such that the quantities $P(\cdot)$ and $K(\cdot)$ are continuous over $B(\theta,\e)$, which is guaranteed by our burn-in condition \cref{eq:burninsfregretlb} combined with Proposition 1 of \cite{mania2019certainty}. Moreover, we may assume that $ \sup_{\theta' \in B(\theta,\e)}\mathsf{R}_T^\pi(\theta) \leq C \sqrt{T}$ for a constant $C>0$ to be determined later.

Before we proceed with the main argument of the proof, we also choose the matrix $W_0 \in \mathbb{R}^{d_\theta \times d}$ as $W_0 = V_0 \Pi_0$ where 
$\Pi_0\in \mathbb{R}^{\dim \mathtt{U} \times d}$ is a free variable and as before the columns of $V_0 \in \mathbb{R}^{d_\theta \times \dim \mathtt{U}}$ span $\mathtt{U}$. Note that by construction the columns of $W_0$ are linear combinations of the columns of $V_0$ and hence elements of $\mathtt{U}$. We also introduce a smooth, compactly supported prior $\mu$ on $\{\theta+W_0 v: \| v\| \leq \e \}$. In the sequel, we write $\dop_\theta$ for Jacobian in $\theta$-space and $\dop_v$ for Jacobian in $v$-space (e.g. of the composite function $\VEC K(\theta+W_0v)$). By restricting $\Pi_0$ to be norm-preserving and then optimizing over $\Pi_0$, the variational characterization of eigenvalues yields that:
\begin{equation}
\label{eq:fishermanipul}
    \begin{aligned}
&\lambda_1 \left( W_0^\T\E \I(\Theta,\pi,T)W_0+\J(\mu) \right)\\
&\leq \lambda_1 \left( W_0^\T\E \I(\Theta,\pi,T)W_0\right) +\lambda_{\max}(\J(\mu))\\
    &\leq\lambda_{1+\dim \mathtt{U}-d} \left(V_0^\T (\E \I(\Theta,\pi,T) )V_0\right)+\lambda_{\max}(\J(\mu))\\
    &\leq \frac{1}{1+\dim \mathtt{U}-d}\tr \left(V_0^\T (\E \I(\Theta,\pi,T) )V_0\right)+\lambda_{\max}(\J(\mu))\\
    &\leq \frac{L R_{T}^\pi(\theta)+(1+\dim \mathtt{U}-d)\lambda_{\max}(\J(\mu))}{1+\dim \mathtt{U}-d}\\
    &\leq \frac{L C \sqrt{T}+(1+\dim \mathtt{U}-d)\lambda_{\max}(\J(\mu))}{1+\dim \mathtt{U}-d}.
    \end{aligned}
\end{equation}

The proof now proceeds by invoking \Cref{lem:regretlbrelaxed} and \eqref{eq:regretlbrelaxedSF} to relax the supremum in the theorem statement to an expectation, and in particular to find that:
\begin{align*}
  &\inf_\pi \sup_{\theta' \in B(\theta,\e)}\mathsf{R}_T^\pi(\theta)\\
    &\geq\sum_{t=0}^{T-1} \E_{\Theta \sim \mu}\E_{\Theta}^\pi \tr \Big[N_\star  \left( \E[K(\Theta)|\mathcal{Y}_{T-1}]-K(\Theta)\right) \left(X_tX_t^\T \right) \left( \E[K(\Theta)|\mathcal{Y}_{T-1}]-K(\Theta) \right)^\T \Big].
\end{align*}
Introduce further a positive semi-definite matrix $\Psi$ and the event 
\begin{align*}
    \mathcal{E}\triangleq \left\{ \sum_{t=0}^{T-1} X_tX_t^\T \succeq \Psi T \right\}.
\end{align*}
By definition of $\mathcal{E}$:
\begin{equation}
\begin{aligned}
\label{eq:calcinsfproof_1}
   & \inf_\pi \sup_{\theta' \in B(\theta,\e)}\mathsf{R}_T^\pi(\theta)\\
    &\geq\sum_{t=0}^{T-1} \E_{\Theta \sim \mu}\E_{\Theta}^\pi \tr \Big[N_\star  \left( \E[K(\Theta)|\mathcal{Y}_{T-1}]-K(\Theta)\right) \left(X_tX_t^\T \right) \left( \E[K(\Theta)|\mathcal{Y}_{T-1}]-K(\Theta) \right)^\T \Big]\\
    &\geq \E_{\Theta \sim \mu}\E_{\Theta}^\pi \tr \Big[N_\star  \left( \E[K(\Theta)|\mathcal{Y}_{T-1}]-K(\Theta)\right) \left( \sum_{t=0}^{T-1} X_tX_t^\T \right) \left( \E[K(\Theta)|\mathcal{Y}_{T-1}]-K(\Theta) \right)^\T \Big]\\
    &\geq T \E_{\Theta \sim \mu}\E_{\Theta}^\pi \tr \Big[N_\star  \left( \E[K(\Theta)|\mathcal{Y}_{T-1}]-K(\Theta)\right) \mathbf{1}_{\mathcal{E}} \Psi \left( \E[K(\Theta)|\mathcal{Y}_{T-1}]-K(\Theta) \right)^\T \Big]\\
    &\geq T \E_{\Theta \sim \mu}\E_{\Theta}^\pi  \left\| \VEC  \left( \E[\sqrt{N_\star}K(\Theta)\sqrt{\Psi}|\mathcal{Y}_{T-1}]-\sqrt{N_\star}K(\Theta)\sqrt{\Psi}\right) \mathbf{1}_{\mathcal{E}} \right\|_2^2.
\end{aligned}
\end{equation}
We now continue by writing the norm (inner product) in \eqref{eq:calcinsfproof_1} as the trace of an outer product:
\begin{equation}
\begin{aligned}\label{eq:calcinsfproof_2}
   \eqref{eq:calcinsfproof_1}  &=  T  \E\tr \Bigg( \VEC \left( \E[\sqrt{N_\star}K(\Theta)\sqrt{\Psi}|\mathcal{Y}_{T-1}]-\sqrt{N_\star}K(\Theta)\sqrt{\Psi}\right)\\
    &\times \VEC \left( \E[\sqrt{N_\star}K(\Theta)\sqrt{\Psi}|\mathcal{Y}_{T-1}]-\sqrt{N_\star}K(\Theta)\sqrt{\Psi}\right)^\T  \mathbf{1}_{\mathcal{E}}\Bigg)\\
    &=
    T \E \tr \left( (\Psi \otimes N_\star) \left(\VEC \left(\E[K(\Theta)|\mathcal{Y}_{T-1}]-\VEC K(\Theta) \right)\left(\E[K(\Theta)|\mathcal{Y}_{T-1}]-\VEC K(\Theta) \right)^\T  \right) \mathbf{1}_{\mathcal{E}}\right)
\end{aligned}
\end{equation}
where the last line uses standard identities relating vectorization and Kronecker products.

Now, since $\E[K(\Theta)|\mathcal{Y}_{T-1}]$ is an estimator of $K(\Theta)$, we may apply \Cref{thm:vtineq}:
\begin{equation}
\label{eq:calc:in:sfproof}
    \begin{aligned}
 \eqref{eq:calcinsfproof_2} &\geq T \tr \Bigg((\Psi \otimes N_\star) \E \left[ (\dop_\theta K(\Theta) W_0  \mathbf{1}_{\mathcal{E}})\right] \\
 &
 \times \left( W_0^\T\E \I(\Theta,\pi,T)W_0+\J(\mu)\right)^{-1}  W_0^\T \E \left[ (\dop_\theta K(\Theta) \mathbf{1}_{\mathcal{E}})^\T\right]   \Bigg).
    \end{aligned}
\end{equation}
Above, \eqref{eq:calc:in:sfproof} uses the fact  that $\Theta$ is a $W_0$-affine translate of $\theta$, hence the Jacobian transforms as $\dop_v \VEC K(V) = \dop_\theta \VEC K(\Theta) W_0$ where $V$ is any random variable such that $\Theta = \theta +W_0 V$. Similar reasoning yields that the information matrix in $\theta$-space can be written as $W_0^\T\E \I(\Theta,\pi,T)W_0$.

Hence, if we combine \eqref{eq:calc:in:sfproof} with \eqref{eq:fishermanipul} we find that:
\begin{equation}
    \begin{aligned}\label{eq:fisherisbddinsfproof}
    &\inf_\pi \sup_{\theta' \in B(\theta,\e)}\mathsf{R}_T^\pi(\theta)\\
    &\geq  \frac{T(1+\dim \mathtt{U}-d)}{L C \sqrt{T}+(1+\dim \mathtt{U}-d)\lambda_{\max}(\J(\mu))}   \mathbf{P}^2(\mathcal{E}) \inf_{\theta',\tilde\theta \in B(\theta,\e)}\tr \Bigg((\Psi\otimes N_\star)(\dop_\theta \VEC(K(\theta'))\\
&\times W_0W_0^\T (\dop_\theta \VEC(K(\theta'))^\T\Bigg).
    \end{aligned}
\end{equation}

To finish the proof, we need to control the event $\mathcal{E}$. We will show that if  for some (universal) polynomial function $\mathsf{poly}$
\begin{align}\label{eq:omitteddisplay}
    T\geq \sup_{\theta'\in B(\theta,\e)}\mathsf{poly}\left(9^{\dx},\opnorm{B},\opnorm{A},\opnorm{Q^{-1}}\opnorm{R},\opnorm{R^{-1}},\opnorm{B},\opnorm{P},\opnorm{K},\opnorm{\Sigma_W},C \right)
\end{align}
where we have omitted dependency on $\theta'$ in the arguments of $\mathsf{poly}$ in \eqref{eq:omitteddisplay}. We show that $\mathbf{P}(\mathcal{E})\geq 1/2$---in fact we establish the stronger statement that $\inf_{\theta}\mathbf{P}_{\theta}(\mathcal{E})\geq 1/2$---with any $\Psi$  satisfying
\begin{align*}
    \Psi \preceq\sum_{j=0}^{T^{1/16}}(A(\theta')+B(\theta')K(\theta'))^j\Bigg[ \Sigma_W  -I_{\dx}T^{-1/8}\Bigg]((A(\theta')+B(\theta')K(\theta'))^\T)^j 
\end{align*}
 for all $\theta' \in B(\theta,\e)$ in \Cref{sec:boundPEN}. Taking this for granted for now, we now choose
\begin{align*}
    C =   \sqrt{\frac{(1+\dim \mathtt{U}-d)}{ 8 L  }}\sqrt{ \inf_{\theta',\tilde\theta \in B(\theta,\e)}\tr \Bigg((\Psi\otimes N_\star)(\dop_\theta \VEC(K(\theta'))
 W_0W_0^\T (\dop_\theta \VEC(K(\theta'))^\T\Bigg)}.
\end{align*}
Hence with this choice of $C$ and as long as $\sqrt{T} \geq\frac{(1+\dim \mathtt{U}-d)\lambda_{\max}(\J(\mu))}{LC}$ we are guaranteed that 
\begin{equation}
    \begin{aligned}\label{eq:fisherisbddinsfproof3}
    &\inf_\pi \sup_{\theta' \in B(\theta,\e)}\mathsf{R}_T^\pi(\theta)\\
    &\geq  \sqrt{T} \sqrt{\frac{1+\dim \mathtt{U}-d}{8L}}\sqrt{ \inf_{\theta',\tilde\theta \in B(\theta,\e)}\tr \Bigg((\Psi\otimes N_\star)(\dop_\theta \VEC(K(\theta'))
 W_0W_0^\T (\dop_\theta \VEC(K(\theta'))^\T\Bigg)}.
    \end{aligned}
\end{equation}
This finishes the proof, since if $\sup_{\theta' \in B(\theta,\e)}\mathsf{R}_T^\pi(\theta) \leq C\sqrt{T}$ does not hold for our choice of $C$, \eqref{eq:fisherisbddinsfproof3} holds trivially. \hfill $\blacksquare$

\subsubsection{Bounding $\mathbf{P}(\mathcal{E})$}
\label{sec:boundPEN}

Fix any $\theta' \in B(\theta,\e)$ and denote $A=A(\theta'),B=B(\theta')$, $P=P(\theta')$, $\mathsf{R}_T^\pi=\mathsf{R}_T^\pi(\theta')$ and $K=K(\theta')$ to avoid cumbersome notation.\footnote{This convention applies to \Cref{sec:boundPEN} alone.} The evolution of $X_tX_t^\T$ under $\theta'$ is given by:
\begin{equation}
\begin{aligned}\label{eq:xtxttopequation}
    X_tX_t^\T &= (AX_{t-1}+BU_{t-1}+W_{t-1})(AX_{t-1}+BU_{t-1}+W_{t-1})^\T\\
              &= ((A+BK)X_{t-1}+B(U_{t-1}-KX_{t-1})+W_{t-1})\\
              &\times((A+BK)X_{t-1}+B(U_{t-1}-KX_{t-1})+W_{t-1})^\T\\
              &=(A+BK)X_{t-1}X_{t-1}^\T(A+BK)^\T\\
              &+(B(U_{t-1}-KX_{t-1})+W_{t-1})(B(U_{t-1}-KX_{t-1})+W_{t-1})^\T \\
              &+\sym (A+BK)X_{t-1}(B(U_{t-1}-KX_{t-1})+W_{t-1})^\T\\
              &\succeq (A+BK)X_{t-1}X_{t-1}^\T(A+BK)^\T+W_{t-1}W_{t-1}^\T\\
              &+\sym B(U_{t-1}-KX_{t-1}))(W_{t-1})^\T \\
              &+\sym (A+BK)X_{t-1}(B(U_{t-1}-KX_{t-1})+W_{t-1})^\T \\
              &=(A+BK)X_{t-1}X_{t-1}^\T(A+BK)^\T+W_{t-1}W_{t-1}^\T\\
              &+\sym\left[ B(U_{t-1}+AX_{t-1}))\right](W_{t-1})^\T \\
              &+\sym (A+BK)X_{t-1}(B(U_{t-1}-KX_{t-1}))^\T
\end{aligned}
\end{equation}
where $\sym M = M+M^\T$ for any matrix $M$. If we iterate \eqref{eq:xtxttopequation} $m$ times, sum over $t$ and exchange the order of summation we obtain:
\begin{equation}
\begin{aligned}\label{eq:xtxttopequation2}
   \sum_{t=0}^{T-1} X_tX_t^\T &=\sum_{j=0}^{m}(A+BK)^j\Bigg[\sum_{t=0}^{T-1}W_{t-j}W_{t-j}^\T\Bigg]((A+BK)^\T)^j\\
              &+\sum_{j=0}^{m}(A+BK)^j\Bigg[\sum_{t=0}^{T-1}\sym\left[ B(U_{t-j}+AX_{t-j}))\right](W_{t-j})^\T\Bigg]((A+BK)^\T)^j \\
              &+\sum_{j=0}^{m}(A+BK)^j\Bigg[\sum_{t=0}^{T-1}\sym (A+BK)X_{t-j}(B(U_{t-j}-KX_{t-j}))^\T\Bigg]((A+BK)^\T)^j.
\end{aligned}
\end{equation}
where quantities with negative indices are taken to be zero.

The strategy is now as follows. We wish to show that the first term in \eqref{eq:xtxttopequation2} is dominant and that it approximates the covariance of the optimal policy. The second term will be bounded by a martingale argument and the third term will be bounded using the hypothesis that $\mathsf{R}_T^\pi \leq C\sqrt{T}$.  To this end, we now introduce the following three events:
\begin{equation}
\begin{aligned}\label{eq:3eventsforsfregret}
    \mathcal{E}_{1} &\triangleq \left\{\exists j : \sum_{t=0}^{T-1}W_{t-j}W_{t-j}^\T \nsucc \Sigma_W- I_{\dx}\opnorm{\Sigma_W} \geq  T^{3/4} \right\},\\
    \mathcal{E}_{2} &\triangleq \left\{\exists j : \bigopnorm{\sum_{t=0}^{T-1}\sym\left[ B(U_{t-j}+AX_{t-j}))\right](W_{t-j})^\T} \geq  T^{3/4} \right\},\\
    \mathcal{E}_{3} &\triangleq \left\{\exists j : \bigopnorm{\sum_{t=0}^{T-1}\sym (A+BK)X_{t-1}(B(U_{t-j}-KX_{t-j}))^\T} \geq T^{7/8}\right\}.
\end{aligned}
\end{equation}
Below we provide bounds on each of these failure-events, which we will now combine to arrive at a high probability lower bound on \eqref{eq:xtxttopequation2}.

By combining the bounds \eqref{eq:boundevent1},\eqref{eq:boundevent2} and \eqref{eq:boundevent3} with the choice $m=T^{1/16}$ we find that there exist a polynomial function $\mathsf{poly}$ such that as long
\begin{align*}
    T\geq \mathsf{poly}\left(9^{\dx},\opnorm{B},\opnorm{A},\opnorm{Q^{-1}}\opnorm{R},\opnorm{R^{-1}},\opnorm{B},\opnorm{P},\opnorm{K},\opnorm{\Sigma_W},C \right)
\end{align*}
we have
\begin{align*}
    \mathbf{P}\left(\sum_{t=0}^{T-1}X_tX_t^\T \succeq\sum_{j=0}^{T^{1/16}}(A+BK)^j\Bigg[ \Sigma_W T -I_{\dx}(\opnorm{\Sigma_W}T^{3/4}+T^{3/4}+T^{7/8})\Bigg]((A+BK)^\T)^j  \right) \geq 1/2.
\end{align*}

\paragraph{Bounding $\mathcal{E}_{1}$}
Fix $\eta>0$ and $s\geq 1$. We invoke \citet[Exercise 4.7.3]{vershynin2018} to obtain that
\begin{align*}
    \mathbf{P} \left(\sum_{t=0}^{T-1} W_{t-j}W_{t-j}^\T \nsucc \Sigma_W T -\eta I_{\dx}\opnorm{\Sigma_W} T   \right)\leq 2\exp(-s^2\dx)
\end{align*}
for every $s$ satsifying $T \geq c(s/\eta)^2\dx$ for a universal constant $c>0$. Set $\eta=T^{-1/4}$, then for a second universal positive constant $c'$:
\begin{align}\label{eq:boundevent1}
    \mathbf{P} \left(\sum_{t=0}^{T-1} W_{t-j}W_{t-j}^\T \nsucc \Sigma_W T - I_{\dx}\opnorm{\Sigma_W} T^{3/4}   \right)\leq 2\exp\left(-c'T^{1/4}\right).
\end{align}

\paragraph{Bounding $\mathcal{E}_{2}$}
Fix $\eta=1/4$ and let $(v_1,v_1',\dots,v_N,v_N')$ be an optimal $\eta$-cover of $\mathbb{S}^{\dx-1} \times \mathbb{S}^{\dx-1}$. By Chebyshev's inequality:  
\begin{align*}
    & \mathbf{P}\left(\sum_{t=0}^{T-1}2v_{i}^\T\left[ B(U_{t-j}+AX_{t-j}))\right](W_{t-j})^\T v_{i}'\geq T^{3/4}\right)\\
    &\leq \frac{\E \left(\sum_{t=0}^{T-1}2v_{i}^\T\left[ B(U_{t-j}+AX_{t-j}))\right](W_{t-j})^\T v_{i}'\right)^2}{T^{3/2}}\\
    &\leq \frac{2\sum_{t=0}^{T-1} \opnorm{B}^2 \E \left( \left[ (U_{t-j}+AX_{t-j}))\right]^\T\left[ (U_{t-j}+AX_{t-j}))\right] \right)\tr \Sigma_W }{T^{3/2}}\\
    &\leq \frac{4 \opnorm{B}^2(\opnorm{A}^2\opnorm{Q^{-1}}\vee \opnorm{R^{-1}} ) [V_{T}^\star+R_{T}^\pi] \tr \Sigma_W }{T^{3/2}}\\
    &\leq  \frac{4 \opnorm{B}^2(\opnorm{\opnorm{A}^2Q^{-1}}\vee \opnorm{R^{-1}} ) [\tr(P\Sigma_W)T+C\sqrt{T}] \tr \Sigma_W }{T^{3/2}}\\
    &\leq \frac{4 \opnorm{B}^2(\opnorm{A}^2\opnorm{Q^{-1}}\vee \opnorm{R^{-1}} ) [\tr(P\Sigma_W)+C] \tr \Sigma_W }{T^{1/2}}.
\end{align*}

By \citet[Exercise 4.4.3]{vershynin2018}, we have
\begin{multline*}
    \bigopnorm{\sum_{t=0}^{T-1}\sym\left[ B(U_{t-j}+AX_{t-j}))\right](W_{t-j})^\T} \\ \leq \frac{1}{1-2\eta} \sup_{v_i,v_i'}\left\|\sum_{t=0}^{T-1}2v_i^\T\left[ B(U_{t-j}+AX_{t-j}))\right](W_{t-j})^\T v_i' \right\|_2.
\end{multline*}
Hence by a union bound it follows that (inserting $\eta=1/4$): 
\begin{multline}\label{eq:boundevent2}
    \mathbf{P} \left\{\exists j : \bigopnorm{\sum_{t=0}^{T-1}\sym\left[ B(U_{t-j}+AX_{t-j}))\right](W_{t-j})^\T} \geq  2 T^{3/4} \right\}\\
    \leq\frac{4 m 9^{2\dx} \opnorm{B}^2(\opnorm{A}^2\opnorm{Q^{-1}}\vee \opnorm{R^{-1}} ) [\tr(P\Sigma_W)+C] \tr \Sigma_W }{T^{1/2}}.
 \end{multline}
\paragraph{Bounding $\mathcal{E}_{3}$}
Fix $\eta=1/4$ and let $(v_1,v_1',\dots,v_N,v_N')$ be an optimal $\eta$-cover of $\mathbb{S}^{\dx-1} \times \mathbb{S}^{\dx-1}$. Let $\mathcal{E}_{3,i}$ be the event that 
\begin{align*}
    \sum_{t=0}^{T-1}2v_i^\T (A+BK)X_{t-1}(B(U_{t-j}-KX_{t-j}))^\T v_{i}'
\end{align*}
is nonnegative. By Markov's inequality and Cauchy-Schwarz:
\begin{align*}
    & \mathbf{P}\left( \sum_{t=0}^{T-1}2v_i^\T (A+BK)X_{t-1}(B(U_{t-j}-KX_{t-j}))^\T v_{i}'\geq T^{7/8}\right)\\
    &=\mathbf{P}\left( \sum_{t=0}^{T-1}2v_i^\T (A+BK)X_{t-1}(B(U_{t-j}-KX_{t-j}))^\T v_{i}' \mathbf{1}_{\mathcal{E}_{3,i}}\geq T^{7/8}\right)\\
    &\leq \frac{\E \left(\sum_{t=0}^{T-1}2v_i^\T (A+BK)X_{t-1}(B(U_{t-j}-KX_{t-j}))^\T v_{i}'\mathbf{1}_{\mathcal{E}_{3,i}} \right)}{T^{7/8}}\\
    &= T \frac{\E \left(\frac{1}{T}\sum_{t=0}^{T-1}2v_i^\T (A+BK)X_{t-1}(B(U_{t-j}-KX_{t-j}))^\T v_{i}'\mathbf{1}_{\mathcal{E}_{3,i}} \right)}{T^{7/8}}\\
    &\leq 2T \frac{\sqrt{\E \left(\frac{1}{T}\sum_{t=0}^{T-1}\|(A+BK)X_{t-1}\|_2^2 \right) }\sqrt{\E \left(\frac{1}{T}\sum_{t=0}^{T-1}\|(B(U_{t-j}-KX_{t-j})) \|_2^2 \right) } }{T^{7/8}}\\
    &\leq 2 \frac{\sqrt{ \opnorm{A+BK}^2 \opnorm{Q^{-1}}^2 (V_T^\star+\mathsf{R}_T^\pi) }\sqrt{\opnorm{B}^2\opnorm{(B^\top P B+R)^{-1}} \mathsf{R}_T^\pi } }{T^{7/8}}\\
    &\leq 2 \frac{\sqrt{ \opnorm{A+BK}^2 \opnorm{Q^{-1}}^2 (T \tr P\Sigma_W +C\sqrt{T}) }\sqrt{\opnorm{B}^2\opnorm{(B^\top P B+R)^{-1}} C\sqrt{T} } }{T^{7/8}}\\
    &\leq 2 \frac{\sqrt{ \opnorm{A+BK}^2 \opnorm{Q^{-1}}^2 ( \tr P\Sigma_W +C) }\sqrt{\opnorm{B}^2\opnorm{(B^\top P B+R)^{-1}} C} }{T^{1/8}}.
\end{align*}
Thus by another union bound:
\begin{multline}\label{eq:boundevent3}
    \mathbf{P}\left\{\exists j : \bigopnorm{\sum_{t=0}^{T-1}\sym (A+BK)X_{t-1}(B(U_{t-j}-KX_{t-j}))^\T} \geq 2 T^{7/8}\right\} \ \\
   \leq 2m 9^{2\dx} \frac{\sqrt{ \opnorm{A+BK}^2 \opnorm{Q^{-1}}^2 ( \tr P\Sigma_W +C) }\sqrt{\opnorm{B}^2\opnorm{(B^\top P B+R)^{-1}} C} }{T^{1/8}}.
\end{multline}

\subsection{Proof of Corollary~\ref{statefeedbackcorr}}
We consider the coordinates 
\begin{equation}
\begin{aligned}
\label{eq:corrconsab}
\VEC A(\theta) &= \VEC A-\VEC [ (\VEC^{-1} \theta) K]\\
\VEC B(\theta) &= \VEC B +\theta
\end{aligned}
\end{equation}

By virtue of Proposition~\ref{simchoprop} we know that the instance is $\e$-locally uninformative (for every $\e >0$) with $d_\theta=\dim \mathtt{U} = \dx\du$. Hence for every $d \leq \dx \du$, we have  by \Cref{thm:sfregretlb} that there exists $W_0$ with $d$-many independent columns that satisfy $\spn W_0 \subset \mathtt{U}$. Note also that for $d = \dx \du$ we have $W_0 W_0^\T = I_{d_\theta}$.

We need to analyze the term:
\begin{align*}
\tr\Bigg( [\Sigma^\star_X(\theta) \otimes    (B^\top(\theta) P(\theta)B( \theta)+R)] \left(\dop_\theta\VEC K( \theta) \right)  W_0W_0^\T\left(\dop_\theta \VEC K( \theta) \right)^\top\Bigg)\\
\end{align*}
For this parameterization, appealing to Proposition~\ref{simchoprop} and  by virtue of \eqref{ch6:simchoder} (see also \citet{simchowitz2020naive}), we have an explicit expression for $\dop_\theta\VEC K( \theta)   $. Namely, 
\begin{align*}
\dop_\theta\VEC K( \theta) =- \left( (A(\theta)+B(\theta)K(\theta))^\top P(\theta) \right) \otimes \left(B^\top(\theta) P(\theta) B(\theta)+R\right)^{-1} .
\end{align*}
Straightforward calculations involving the trace cyclic property and the Lyapunov identity $(A+BK_\star) \Sigma_X^\star(A+BK_\star)^\T = \Sigma_X^\star -I_{d_{\mathsf{X}}}  $ now yield
\begin{multline}\label{eq:corr:fcalc}
    \tr\Bigg( [\Sigma^\star_X(\theta) \otimes    (B^\top(\theta) P(\theta)B( \theta)+R)] \left(\dop_\theta\VEC K( \theta) \right)  W_0W_0^\T\left(\dop_\theta \VEC K( \theta) \right)^\top\Bigg)\\
    =  \tr\Bigg( \left[\left(P(\theta)[\Sigma^\star_X(\theta)-I_{\dx}]P(\theta)\right) \otimes    \left(B^\top(\theta) P(\theta)B( \theta)+R\right)^{-1}\right] W_0W_0^\T\Bigg).
\end{multline}

Using Lemma~\ref{statefeedbackinfocomp} we may choose \begin{align*}
L=\tr(\Sigma^{-1}_w)\left( \sup_{\bar\theta\in B(\theta,\e)} \| \dop_\theta[A(\bar \theta) \:  B(\bar \theta)] \|_\infty^2\right) \|(B^\top P(\theta)B+R)^{-1}\|_{\mathsf{op}}
\end{align*}
with $\e$ arbitrary.  Using (\ref{eq:corrconsab}), we have $\|\dop_\theta[A(\bar \theta) \:  B(\bar \theta)] \|_{\mathsf{op}}^2 \leq 1\vee \| K(\theta)K^\top(\theta)\|_{\mathsf{op}}$. Next, note that
\begin{align*}
 \|(B^\top P(\theta)B+R)^{-1}\|_{\mathsf{op}} = \frac{1}{\sigma_{\min}(B^\top(\theta)  P(\theta)B(\theta)+R)},
\end{align*}
that $\tr \Sigma^{-1}_w \leq \sigma_{\max}(\Sigma^{-1}_w) \dx = \frac{\dx}{\sigma_{\min}(\Sigma_W)}$.  Hence, we may take
\begin{align}
\label{eq:Lincorrsf}
L \leq \frac{\dx (1\vee \| K(\theta)K^\top(\theta)\|_{\mathsf{op}})}{\sigma_{\min} (\Sigma_W) \times \sigma_{\min} (B^\top(\theta)  P(\theta)B(\theta)+R) }.
\end{align}
By combining \eqref{eq:corr:fcalc} with \eqref{eq:Lincorrsf} and invoking \Cref{thm:sfregretlb} we obtain:
\begin{multline*}
    \liminf_{T\to \infty} \sup_{ \substack{A' , B': \\ \| [A'-A \: B'-B] \|_{\mathsf{op}} \leq T^{-\alpha}}} \frac{ \mathsf{R}_T^\pi(A',B')}{\sqrt{T}}\\
\geq c \sqrt{
 \frac{(1+\dx\du-d)\lambda_{\min}(\Sigma_W) \tr \Bigg(\left[\left(P(\theta)[\Sigma^\star_X(\theta)-I_{\dx}]P(\theta)\right)\otimes I_{\du}\right]W_0W_0^\T \Bigg)}{\dx(1+ \opnorm{K(\theta)K^\T(\theta)})}}
\end{multline*}
for some universal positive constant $c>0$. Hence for $d=\dx\du$ we have that:
\begin{equation*}
    \liminf_{T\to \infty} \sup_{ \substack{A' , B': \\ \| [A'-A \: B'-B] \|_{\mathsf{op}} \leq T^{-\alpha}}} \frac{ \mathsf{R}_T^\pi(A',B')}{\sqrt{T}}
\geq c \sqrt{
 \frac{ \du  \lambda_{\min}(\Sigma_W) \tr \Bigg((P(\theta)[\Sigma^\star_X(\theta)-I_{\dx}]P(\theta)\Bigg)}{\dx(1+ \opnorm{K(\theta)K^\T(\theta)})}}
\end{equation*}
For $d<\dx\du$, since $W_0W_0^\T$ is an orthogonal projector onto a $d$-dimensional subspace, it also follows that:
\begin{multline*}
    \liminf_{T\to \infty} \sup_{ \substack{A' , B': \\ \| [A'-A \: B'-B] \|_{\mathsf{op}} \leq T^{-\alpha}}} \frac{ \mathsf{R}_T^\pi(A',B')}{\sqrt{T}}\\
\geq c \sqrt{
 \frac{d(1+\dx\du-d) \lambda_{\min}(\Sigma_W) \lambda_{ \min}\left(P(\theta)[\Sigma^\star_X(\theta)-I_{\dx}]P(\theta)\right) }{\dx(1+ \opnorm{K(\theta)K^\T(\theta)})}}.
\end{multline*}
The proof is completed by optimizing over $d$. \hfill $\blacksquare$

\subsection{Proof of \Cref{statefeedbackinfocomp}}
\label{sec:proof:statefeedbackinfocomp}
Fix $\theta'\in B(\theta,\e)$ and let the columns of $V_0$ span the information singular subspace $\mathtt{U}$. We now introduce a quadratic potential, $f:\mathbb{R}^{T(\dx+\du)}\to \mathbb{R}$, in terms of the dummy variables $\eta_j=(\eta_j^1,\eta_j^2) \in \mathbb{R}^{\dx+\du}$ as 
\begin{align}
\label{eq:thepotential}
f(\eta_{0:T-1}) = \tr \left[V_0^\T\left(\sum_{j=1}^{T-1}  \left( \dop_\theta[A(\theta') \:  B(\theta')]\right)^\T ( \eta_{j-1} \eta_{j-1}^\T \otimes \Sigma_{W}^{-1}(\theta)) \left( \dop_\theta[A(\theta') \:  B(\theta')]\right) \right) V_0 \right].
\end{align}
Observe that $\E_{\theta'}^\pi f(Z_{0:T-1}) =  \tr V_0^\T \I(\pi;\theta',T) V_0$. Our next observation is that the restriction of $f$ to the subspace $\eta^2_j = K(\theta)\eta^1_j, \forall j$ is identically zero by uniformativeness and invoking \Cref{prop:charpropsf} (recall that the columns of $V_0$ span $\mathtt{U}$).

Since the linear manifold $F_0 =\{\eta^2_j = K(\theta)\eta^1_j, \forall j\}\subset \mathbb{R}^{ T (\dx+\du)}$ is a global minimum for $f$, we may, for any fixed choice of $\eta^1_j,0=1,\dots,T-1$, Taylor-expand $f$ around such a point with coordinate $\eta^1$ fixed, to obtain

\begin{align*}
f(\eta_{0:T-1}) \leq \frac{1}{2} \opnorm{\nabla^2_{\eta^{2}_{0:T-1}} f}  \left \|(\eta^2_j-K(\theta)\eta^1_j)_{j=0}^{T-1} \right\|_2.
\end{align*}
where we have used the fact that $f$ is convex quadratic function with minimum $0$, attained at all points in $F_0$, so that its taylor-expansion around such a point is just a quadratic form. 

By introducing a factor $I_{\du} = (B^\T P(\theta)B+R)^{-1}(B^\T P(\theta)B+R)$ we obtain
\begin{align*}
f(\eta_{0:T-1})& \leq\frac{1}{2} \opnorm{(B^\T P(\theta)B+R)^{-1}} \| \nabla^2_{\eta^{2}_{0:T-1}} f\|_{\mathsf{op}}\\
&\times \sum_{t=0}^{T-1}  (\eta^2_j-K(\theta)\eta_j^1)^\T (B^\T P(\theta)B+R) (\eta_j^2-K(\theta)\eta_j^1)
\end{align*}
In particular, by taking expectations, we have that
\begin{align*}
\tr V_0^\T \I(\pi;\theta',T) V_0 =\frac{1}{2} \E_{\theta'}^\pi f(Z_{0:T-1})  \leq  \opnorm{(B^\T P(\theta)B+R)^{-1}} \opnorm{ \nabla^2_{\eta^{2}_{0:T-1}} f} \mathsf{R}_T^\pi(\theta)
\end{align*}
Since
\begin{align*}
\opnorm{\nabla^2_{\eta^{2}_{0:T-1}} f} \leq \opnorm{ \dop_\theta[A(\theta') \:  B(\theta')] \|}^2 \tr ( \Sigma_W^{-1})
\end{align*}
we have for any $\theta'\in B(\theta,\e)$
\begin{align*}
\tr V_0^\T \I(\pi;\theta',T) V_0 \leq \tr ( \Sigma_W^{-1}) \opnorm{\dop_\theta[A( \theta') \:  B(\theta')] }^2\opnorm{(B^\T P(\theta)B+R)^{-1}} \mathsf{R}_T^\pi(\theta)
\end{align*}
and the result follows.  \hfill $\blacksquare$

%% file: Appendix/VanTrees.tex
\section{Fisher Information and  Van Trees' Inequality}

This appendix collects some facts about Fisher Information and Bayesian Estimation that will be useful in the sequel. Indeed, many---if not most---information-theoretic lower bounds rely on relaxing a minimax complexity by a Bayesian complexity (i.e., weak duality). Instead of fixing the parameter $\theta$, we let $\Theta$ be a random vector taking values in $\mathbb{R}^{d_\Theta}$ and suppose that it has density $\mu$. More concretely, we will be using this in conjunction with the lower bound: 
\begin{align*}
    \inf_{\pi} \sup_{\theta} R^\pi_T(\theta) \geq \inf_{\pi} \E_{\Theta\sim\mu} R^\pi_T(\Theta).
\end{align*}

Let us now consider a more general situation in which we are given an observation $Y$, a random vector taking values in $\mathbb{R}^n$ with (conditional) density $p_\theta(\cdot)=p(\cdot|\theta)$. The following two quantities are key to measuring estimation performance of $\Theta$ from the sample $Y$ with respect to square loss:
\begin{align}
\label{eq:fisherdef}
\I_p(\theta) &=\int  \left( \frac{\nabla_\theta p(y|\theta)}{p(y|\theta)}\right)\left( \frac{\nabla_\theta p(y|\theta)}{p(y|\theta)}\right)^\T p(y|\theta) dy, \textnormal{ and} \\
\label{eq:fisherlocdef}
\J(\mu)&= \int \left( \frac{\nabla_\theta \mu(\theta)}{\mu(\theta)}\right) \left( \frac{\nabla_\theta \mu(\theta)}{\mu(\theta)}\right)^\T  \mu(\theta)  d\theta.
\end{align}
The first quantity \eqref{eq:fisherdef} can be thought of as to (locally, with respect to square loss) measure the information content of the sample $Y$ with regards to $\Theta$. The second quantity \eqref{eq:fisherlocdef} essentially measures the concentration of the prior density $\mu$. See \cite{ibragimov2013statistical} for further details about these integrals and their existence.

Before we state Van Trees' Inequality we need to establish the following Gramian analogue of the Cauchy-Schwarz inequality.

\begin{lemma}
\label{gcslemma}
Fix two random vectors  $V_1, V_2 \in \mathbb{R}^d$ and suppose that $ 0\prec \E V_2V_2^\T \prec \infty$. Then 
\begin{align}
\label{gcs}
\E V_1 V_1^\T \succeq\E V_1V_2^\T (\E V_2V_2^\T)^{-1} \E V_2 V_1^\T. 
\end{align}
\end{lemma}
For scalar random variabes $V_1, V_2$  this reduces to the Cauchy-Schwarz inequality in the space of square integrable random variables. 
\begin{proof}

Take two random vectors $V_1, V_2 \in \mathbb{R}^d$ and suppose that $ 0\prec \E V_2V_2^\T \prec \infty$. Observe that 
\begin{align}
\label{posdefoprod}
0\preceq
\E \begin{bmatrix}
V_1 \\ V_2
\end{bmatrix}
\begin{bmatrix}
V_1^\T & V_2^\T
\end{bmatrix}
=
\begin{bmatrix}
\E V_1V_1^\T & \E V_1 V_2^\T \\
\E V_2V_1^\T &  \E V_2 V_2^\T
\end{bmatrix}.
\end{align}
Since $\E V_2 V_2^\T \succ 0$, (\ref{posdefoprod}) implies, by the Schur complements necessary condition for positive semi-definiteness, that $\E V_1 V_1^\T -\E V_1V_2^\T (\E V_2V_2^\T)^{-1} \E V_2 V_1^\T \succeq 0$. 
\end{proof}

Let us also introduce a function $\psi : \mathbb{R}^{d_\Theta} \to \mathbb{R}^n$. The purpose is to relax the setup above slightly, in that we seek to establish lower bounds for estimating $\psi(\theta)$ instead of just $\theta$. We impose the following regularity conditions:
\begin{enumerate}
\item[B1.] $\mu \in \mathscr{C}^\infty_c(\mathbb{R}^{d_\Theta})$; the prior is smooth with compact support.
\item[B2.] $p(y|\cdot)$ is continuously differentiable on the domain of $\mu$ for almost every $y$.
\item[B3.] The score has mean zero; $ \int \left( \frac{\nabla_\theta p(y|\theta)}{p(y|\theta)}\right) p(y|\theta) dy =0$.
\item[B4.] $\J(\mu)$ is finite and $\I_p(\theta)$ is a continuous function of $\theta$ on the domain of $\mu$.
\item[B5.] $\psi$ is differentiable on the domain of $\mu$.
\end{enumerate}

The following theorem is a less general adaption from \cite{bobrovsky1987some}  which suffices for our needs.

\begin{theorem}
\label{thm:vtineq}
Fix two random variables $(Y,\Theta) \sim   p(\cdot|\cdot) \mu(\cdot)$  and suppose that B1-B5 hold. Let further $\mathcal{E}$ be a $\sigma(Y)$-measurable event. Then for any $\sigma(Y)$-measurable $\hat \psi$:
\begin{align}
\label{VTineq}
\E \left[ (\hat \psi -\psi(\Theta))(\hat \psi -\psi(\Theta))^\T \mathbf{1}_\mathcal{E}\right] \succeq  \E[ \nabla_\theta \psi(\Theta)\mathbf{1}_\mathcal{E}]^\T \left[ \E \I_p(\Theta)+\J(\mu) \right]^{-1} \E [\nabla_\theta \psi(\Theta)\mathbf{1}_\mathcal{E}].
\end{align}
\end{theorem}

\begin{proof}
We invoke \eqref{gcs} with $V_1 = (\hat \psi (Y)-\psi(\Theta))\mathbf{1}_\mathcal{E}$ and $V_2 =\nabla_\theta \log [p(Y|\Theta) \mu(\Theta)] $. We first compute
\begin{equation*}
\begin{aligned}
\E V_2 V_2^\T &= \E\left[ \nabla_\theta \log [p(y|\Theta) \mu(\Theta)] \left(\nabla_\theta \log [p(y|\Theta) \mu(\Theta)]\right)^\T\right]\\
&= \E \left(\frac{\mu(\theta)\nabla_\theta(p(y|\Theta) + p(y|\Theta) \nabla_\theta \mu(\Theta)}{p(y|\Theta)\mu(\Theta)} \right)\left(\frac{\mu(\Theta)\nabla_\theta(p(y|\Theta) + p(y|\Theta) \nabla_\theta \mu(\Theta)}{p(y|\Theta)\mu(\Theta)} \right)^\T\\
&=\E \left( \frac{\nabla_\theta p(y|\Theta)}{p(y|\Theta)}\right)\left( \frac{\nabla_\theta p(y|\Theta)}{p(y|\Theta)}\right)^\T +\E \left( \frac{\nabla_\theta \mu(\Theta)}{\mu(\Theta)}\right) \left( \frac{\nabla_\theta \mu(\Theta)}{\mu(\Theta)}\right)^\T+C_{p\mu}
\end{aligned}
\end{equation*}
where
\begin{align*}
C_{p\mu} &= \E \left( \frac{\nabla_\theta p(y|\Theta)}{p(y|\Theta)}\right)\left( \frac{\nabla_\theta \mu(\Theta)}{\mu(\Theta)}\right)^\T +\E \left( \frac{\nabla_\theta \mu(\Theta)}{\mu(\Theta)}\right)\left( \frac{\nabla_\theta p(y|\Theta)}{p(y|\Theta)}\right)^\T\\
&=0
\end{align*}
in light of B3. Hence
\begin{align}
\label{v2v2eqn}
\E V_2 V_2^\T = \E \I(\theta)+\J(\mu).
\end{align}
We still need to establish that the matrix in (\ref{v2v2eqn}) has full rank. We will see that this is a byproduct of the computation of $\E V_1 V_2^\T$. Now since $\mu\in \mathscr{C}^\infty_c(\mathbb{R}^{d_\Theta})$, we may integrate by parts:
\begin{multline}
\begin{aligned}
\label{ibp}
\E V_1 V_2^\T &=\E \left[ (\hat \psi (Y)-\psi(\Theta))\mathbf{1}_\mathcal{E}\left(\frac{\nabla_\theta p(Y,\Theta)}{p(Y,\Theta)}  \right)^\T \mathbf{1}_\mathcal{E} \right]\\
&=\int\int (\hat \psi (y)-\psi(\theta))\mathbf{1}_\mathcal{E}(y)\left(\frac{\nabla_\theta p(y,\theta)}{p(y,\theta)}  \right)^\T p(y,\theta) dy d\theta\\
&=\int\int \left[\nabla_\theta \psi(\theta)\right]^\T\mathbf{1}_\mathcal{E}(y) p(y,\theta) dy d\theta = \E [\nabla_\theta \psi(\theta)\mathbf{1}_\mathcal{E}]^\T && (\mu \in C^\infty_c(\mathbb{R}^{d_\Theta}) )
\end{aligned}
\end{multline}
where $p(y,\theta) = p(y|\theta)\mu(\theta)$.

In particular, using $\psi(\theta) = \theta$ and choosing $\mathcal{E}$ as the entire sample space yields $\E V_1V_2^\T =I$, which is sufficient to conclude that $\E V_2 V_2^\T$ has full rank. To see this, let $w\in \ker \E V_2 V_2^\T$. Then: $ \E w^\T V_2 V_2^\T w  =0$  and so $V_2^\T w =0$ almost surely. However $\E V_1V_2^\T w \neq 0$ for all $w\neq 0$. Hence $\ker \E V_2 V_2^\T =\{0\}$. The result now follows by (\ref{gcs}) combined with (\ref{v2v2eqn}) and (\ref{ibp}).
\end{proof}

\subsection{Further Properties}
We collect here a few useful properties about \eqref{eq:fisherdef} that will come in hand in the sequel. First, we will give explicit forms of the information quantities just discussed for Gaussian densities. The following lemma is a consequence of Theorem 2.1 in \cite{miller1974complex},  Chapter V.
\begin{lemma}
\label{lem:Gaussianfisher}
Let $\mu(\theta): \mathbb{R}^{d_\theta} \to \mathbb{R}^d$, and $\Sigma(\theta) : \mathbb{R}^{d_\Theta} \to \mathbb{R}^{d\times d} $, with $\Sigma(\theta)\succ 0$ and define $\gamma_\theta(x) = \frac{1}{\sqrt{(2\pi)^d |\Sigma(\theta)|}} \exp\left(-\frac{1}{2} (x-\mu(\theta))^\T\Sigma^{-1}(\theta)(x-\mu(\theta)) \right)$. Then
\begin{align}
\label{Gaussianfishereq}
\I_{\gamma}(\theta) = \underbrace{[\nabla_\theta \mu(\theta)][\Sigma(\theta)]^{-1} [\nabla_\theta \mu(\theta)]^\T}_{\I_{\gamma,\mu}(\theta)} +\underbrace{\frac{1}{2}\tr\left(\Sigma^{-1} (\del_{\theta_m}\Sigma)\Sigma^{-1}\del_{\theta_n}\Sigma\right)_{m,n}  }_{ \I_{\gamma,\Sigma}(\theta)}
\end{align}
where $m \in [d_\Theta], n\in [d]$.
\end{lemma}

The second term can also be written as an explicit matrix using vectorization. Notice that
\begin{align*}
\tr\left(\Sigma^{-1} (\del_{\theta_m}\Sigma)\Sigma^{-1}\del_{\theta_n}\Sigma\right)&= [\VEC (\Sigma^{-1} (\del_{\theta_m}\Sigma))]^\T\VEC (\Sigma^{-1} (\del_{\theta_n}\Sigma))\\
&=[(I_d\otimes \Sigma^{-1})\VEC  (\del_{\theta_m}\Sigma)]^\T[(I_d\otimes \Sigma^{-1})\VEC  (\del_{\theta_n}\Sigma)]\\
&=\VEC  (\del_{\theta_m}\Sigma)^\T (I_d \otimes \Sigma^{-2}) \VEC  (\del_{\theta_n}\Sigma)
\end{align*}
It follows that
\begin{align*}
\frac{1}{2}\tr\left(\Sigma^{-1} (\del_{\theta_m}\Sigma)\Sigma^{-1}\del_{\theta_n}\Sigma\right)_{m,n} &= \frac{1}{2} [\dop_\theta \VEC\Sigma(\theta)]^\T(I_d\otimes \Sigma^{-2}(\theta)) \dop_\theta \VEC\Sigma(\theta)
\end{align*}
so that
\begin{align*}
\I_{\gamma}(\theta) = [\dop_\theta \mu(\theta)]^\T[\Sigma(\theta)]^{-1} \dop_\theta \mu(\theta)+\frac{1}{2} [\dop_\theta \VEC\Sigma(\theta)]^\T(I_d\otimes \Sigma^{-2}(\theta)) \dop_\theta \VEC\Sigma(\theta).
\end{align*}

\paragraph{The Chain Rule}
We next recall the chain rule. Consider the Fisher information as defined in \eqref{eq:fisherdef} for a bivariate density $p_\theta(x,y)=p(x,y|\theta)$. Define the conditional Fisher information as
\begin{align*}
\I_{p(x|y)}(\theta) \triangleq \int \int \nabla_\theta \{  \log p_\theta(x|y)\} \left[\nabla_\theta\{  \log p_\theta(x|y)\} \right]^\T p_\theta(x|y)dx p_\theta(y)  dy
\end{align*}
Suppose that the scores $\nabla_\theta \log p_\theta(y)$ and  $\nabla_\theta \log p_\theta(x|y)$  have mean zero (under $ p_\theta(y)dy$ and for every $y$ under $p_\theta(x|y)dx$ respectively). Then:
\begin{multline}
\begin{aligned}
\label{eq:fichain}
\I_{p(x,y)}(\theta)&= \int \int \nabla_\theta \log p_\theta(x,y)\left[\nabla_\theta \log p_\theta(x,y)\right]^\T p_\theta(x,y) dxdy \\
 &= \int \int \nabla_\theta \log \{ p_\theta(x|y) p_\theta(y)\} \left[\nabla_\theta \log \{p_\theta(x|y) p_\theta(y)\} \right]^\T p_\theta(x|y) p_\theta(y)  dxdy\\
 &=\int \int \nabla_\theta \{  \log p_\theta(x|y)+\log p_\theta(y)\} \left[\nabla_\theta \{  \log p_\theta(x|y)+\log p_\theta(y)\} \right]^\T p_\theta(x|y) p_\theta(y)  dxdy\\
 &=\int \int \nabla_\theta \{  \log p_\theta(x|y)\} \left[\nabla_\theta\{  \log p_\theta(x|y)\} \right]^\T p_\theta(x|y)dx p_\theta(y)  dy\\
 &+\int \int \nabla_\theta \{  \log p_\theta(y)\} \left[\nabla_\theta  \{  \log p_\theta(y)\} \right]^\T   p_\theta(x|y)   p_\theta(y)dydx\\
 &+\textnormal{mixed terms linear in the expectation of score functions}\\
 &=\I_{p(x|y)}(\theta) + \I_{p(x)}(\theta).
 \end{aligned}
\end{multline}

The following lemma establishes sufficient regularity for the score function to have mean zero. We state the result for an unconditional density $p_\theta(\cdot)$ (otherwise apply the result conditionally) that satisfies the following regularity conditions:
\begin{enumerate}
\item[F1.] The densities $p_\theta(y)$ are $y$-almost everywhere  differentiable in $\theta$.
\item[F2.]  The gradients $\nabla_\theta p_\theta(y)$ are absolutely integrable. 
\item[F3.]  The model has finite fisher information; $\tr \I_p(\theta) < \infty$. 
\end{enumerate}

The technical significance of F1-F3 is that they permit the use of the Leibniz rule for exchange of differentiation and integration.

\begin{lemma}
\label{lem:fisherreg}
Assume that F1-F3 hold, then any measurable function $T$ of $y$ with $\E_\theta T^2(y) < \infty$ satisfies
\begin{align}
\label{ch2:leibniz}
\nabla_\theta \int T(y) p_\theta(y) dy =  \int T(y)  \nabla_\theta p_\theta(y) dy = \int T(y) [\nabla_\theta \log p_\theta(y)] p_\theta(y) dy.
\end{align}
\end{lemma}

\begin{proof}
Note that Lebesque integrability of $p_\theta(y)$ and $T(y) p_\theta(y)$ are immediate since $p_\theta(y)$ is a density and  $\E_\theta T^2(y) < \infty$. Together with this observation F1 and F2 are sufficient to conclude by dominated convergence that a Leibniz rule applies to $p_\theta(y)$. Hence, we need to establish that corresponding conditions apply to $T(y) p_\theta(y)$. Obviously, $T(y) p_\theta(y)$ is $y$-almost everywhere  differentiable in $\theta$ since $T$ is independent of $\theta$. It thus suffices to prove the absolute integrability of $T(y) \nabla_\theta p_\theta(y)$. To establish this, write
\begin{equation*}
\begin{aligned}
\|T(y) \nabla_\theta p_\theta(y) \|_{L^1(dy)} &= \|T(y) [\nabla_\theta \log p_\theta(y)] p_\theta(y) \|_{L^1(dy)} \\
&= \|\sqrt{p_\theta(y)} T(y) [\nabla_\theta \log p_\theta(y)] \sqrt{p_\theta(y)} \|_{L^1(dy)}\\
&\leq \|\sqrt{p_\theta(y)} T(y) \|_{L^2(dy)} \| [\nabla_\theta \log p_\theta(y)] \sqrt{p_\theta(y)} \|_{L^2(dy)} \\
&= \sqrt{\E_\theta T^2(y) \tr \I_q(\theta)}
\end{aligned}
\end{equation*}
where the inequality follows by Cauchy-Schwarz. The RHS is finite by hypothesis F3.
\end{proof}

We state \Cref{lem:fisherreg} for completeness, but remark that we are mostly interested in the case $T=1$. In particular, for $T=1$, in which case only F1 and F2 are needed, \eqref{ch2:leibniz} becomes
\begin{align}
\label{ch2:scoremeanzero}
\int \nabla_\theta \log p_\theta(y)  p_\theta (y) dy = \int \frac{\nabla_\theta p_\theta (y)}{p_\theta (y)} p_\theta(y) dy =\int \nabla_\theta p_\theta (y) dy = \nabla_\theta \int p_\theta(y) dy=0.
\end{align}
In principle, this implies that under rather mild regularity conditions, the chain rule for Fisher information \eqref{eq:fichain} can be used.